\documentclass[DIV=calc,10pt,twocolumn]{article} 
\usepackage{simpleConference}
\usepackage{times}
\usepackage{graphicx,float}
\usepackage{amssymb}
\usepackage{url}
\usepackage{caption,lipsum}
\usepackage{abstract}

\usepackage{amsmath}
\usepackage{amsthm}
\usepackage{booktabs} 
\usepackage{wrapfig}
\usepackage{tabularx}

\newtheorem{theorem}{Theorem}
\newtheorem{Lemma}{Lemma}

\newtheorem{proposition}{Proposition}

\newtheorem{definition}{Definition}

\newcommand\question[1]

\newcommand{\set}[1]{\left\{#1\right\}}

\newcommand{\Real}{\mathbb R}

\newcommand{\too}{\rightarrow}

\newcommand{\wt}[1]{\widetilde{#1}} 
\newcommand{\wh}[1]{\widehat{#1}} 

\def \aff{\mathrm{aff}} 
\def \mM{\mathcal{M}} 
\def \mS{\mathcal{S}} 
\def \mP{\mathcal{P}} 
\def \mT{\mathcal{T}} 
\def \mE{\mathcal{E}} 
\def \mF{\mathcal{F}} 
\def \mV{\mathcal{V}} 

\def \etal{{et al}.}
\newcommand{\eg}{{\it e.g.}}
\newcommand{\ie}{{\it i.e.}}

\begin{document}
\title{Multi-chart Generative Surface Modeling}

\author{Heli Ben-Hamu \and Haggai Maron \and Itay Kezurer \and Gal Avineri \and Yaron Lipman\\}
\affiliation{Weizmann Institute of Science}









\twocolumn[{%
	\renewcommand\twocolumn[1][]{#1}%
	\maketitle
	\begin{center}
		\vspace{-0.9cm}
		\includegraphics[width=7.0in]{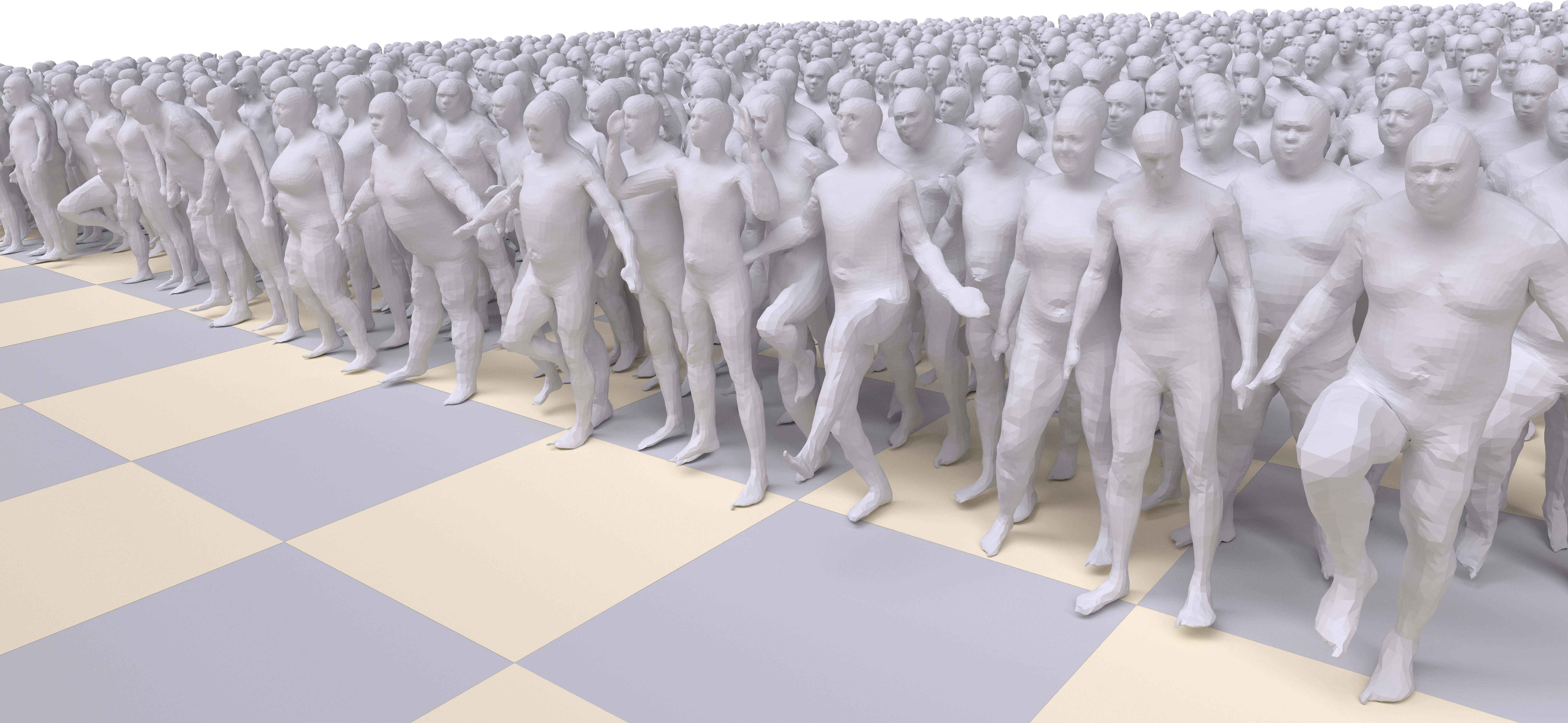}	\captionof{figure}{Our method is able to learn shape distribution and generate unseen shapes. This figure shows 1024 human models randomly generated by our method. }\label{fig:teaser}
	\end{center}%
}]
\saythanks

\subsection*{\centering Abstract} 

This paper introduces a 3D shape generative model based on deep neural networks. 
A new image-like (\ie, tensor) data representation for genus-zero 3D shapes is devised. It is based on the observation that complicated shapes can be well represented by multiple parameterizations (charts), each focusing on a different part of the shape. The new tensor data representation is used as input to Generative Adversarial Networks for the task of 3D shape generation. 

The 3D shape tensor representation is based on a multi-chart structure that enjoys a shape covering property and scale-translation rigidity. Scale-translation rigidity facilitates high quality 3D shape learning and guarantees unique reconstruction. The multi-chart structure uses as input a dataset of 3D shapes (with arbitrary connectivity) and a sparse correspondence between them. The output of our algorithm is a generative model that learns the shape distribution and is able to generate novel shapes, interpolate shapes, and explore the generated shape space.   
%
%
The effectiveness of the method is demonstrated for the task of anatomic shape generation including human body and bone (teeth) shape generation.	



\section{Introduction}


Generative models of 3D shapes facilitate a wide range of applications in computer graphics such as automatic content creation, shape space analysis, shape reconstruction and modeling. 


The goal of this paper is to devise a new (deep) 3D generative model for genus-zero surfaces based on Generative Adversarial Networks (GANs) \cite{goodfellow2016nips}. The main challenge in 3D GANs compared to image GANs is finding a representation of 3D shapes that enables efficient learning. Since standard CNNs work well with image-like data, \ie, tensors, and on the other hand defining CNN on unstructured data seems to pose a challenge \cite{bronstein2017geometric}, most 3D GANs methods concentrate on representing the input shapes in a tensor form. For example, representing the shape using a volumetric grid \cite{wu2016learning,tatarchenko2017octree} or depth-maps \cite{tatarchenko2016multi}. Although natural, these representations suffer from either the high dimensionality of volumetric tensors, their crude brick-like approximation properties,  or the partial, discontinuous and/or occluded cover achieved with projection based methods. 
In a recent paper, Groueix \cite{Groueix18} represent 3D shapes using multiple charts, where each individual chart is defined as a multilayer perceptron (MLP). 

The approach taken in this paper toward 3D shape representation also uses multiple charts but in contrast to previous work the different charts are represented as a \emph{single tensor} (\ie, regular grid of numbers) with the following properties: (i) the different charts are related by a so-called \emph{multi-chart structure} describing their inter relations; (ii) the charts participating in the tensor are smooth (in fact, angle-preserving), bijective, and consistent; and (iii) standard tensor convolution used in off-the-shelf CNNs can be applied to this tensor representation and is equivalent to a well defined convolution on (a cover of) the original surface.


\begin{figure}
\includegraphics[width=\columnwidth]{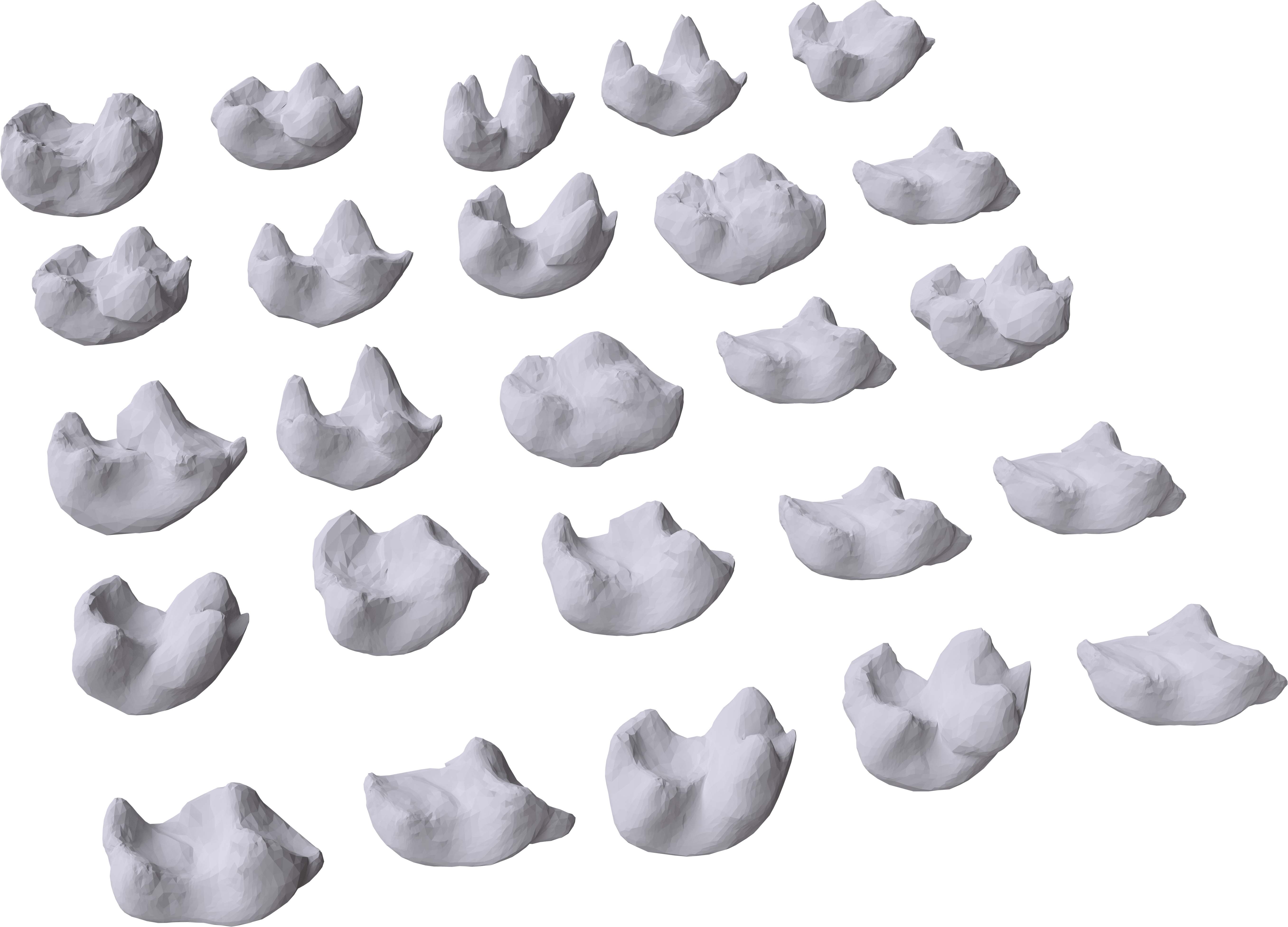}
  \caption{Automatic random generation of 25 teeth models.}\label{fig:teaser_teeth}
\end{figure}

The multi-chart structure is the main building block of our 3D shape tensor representation. Intuitively, a multi-chart structure is a collection of conformal charts, each defined using a triplet of points that, together, cover with small distortion all parts of the shape and are \emph{scale-translation rigid}. Scale-translation (s-t) rigidity is a property that allows recovering the mean and scale of all the charts in a unique manner. S-t rigidity turns out to be significant as the training process has to be performed on normalized charts for effective 3D shape learning. We study s-t rigidity of the multi-chart structure showing it is a generic property and providing a simple sufficient condition for it. 
The multi-chart structure requires only a sparse set of correspondences as input and allows processing shapes with different connectivity and unknown dense correspondence using standard image GAN frameworks.



 We tested our multi-chart 3D GAN method on two classes of shapes: human body and anatomical bone surfaces (teeth). For human body shapes we used datasets of human bodies \cite{bogo2017dynamic,yang2014semantic} consisting of different humans in a collection of poses as input to our 3D GAN framework. Figure \ref{fig:teaser} shows rendering of 1024 models randomly generated using the trained multi-chart 3D GAN. Note the diversity of the human body shapes and poses created by the generative model. For bone surfaces we used the teeth dataset in \cite{boyer2011algorithms}; Figure \ref{fig:teaser_teeth} depicts $25$ teeth randomly generated using our method. As we demonstrate in this paper, our method compares favourably to different baselines and previous methods for 3D shape generation. 
 
 \noindent The code is available at the project webpage\footnote{\url{http://github.com/helibenhamu/multichart3dgans}}.
 
%
%

\section{Previous Work}

\paragraph{Generative adversarial networks.}
Generative adversarial networks (GANs) are deep neural networks aimed at generating data from a given distribution \cite{goodfellow2014generative}.
GANs are composed of two sub-networks (often convolutional neural networks): a generator, which is in charge of generating an instance from the distribution of interest and a discriminator that tries to discriminate between instances that were sampled from the original distribution and instances that were generated by the generator. The training process alternates between optimizing the discriminator to recognize the true samples, and optimizing the generator to fool the discriminator.
These models have become very popular in the last few years and were used to generate many data types such as images \cite{goodfellow2014generative}, videos \cite{vondrick2016generating}, 3D data (as will be reviewed below) and more. Our work uses GANs in order to generate surfaces of a certain class. We will dedicate the rest of this section to generation of 3D data. For further details on GANs see \cite{goodfellow2016nips}.

\paragraph{Volumetric data generation.}
A natural way to use deep learning for 3D data generation is to work on volumetric grids and corresponding volumetric convolutions \cite{wu2016learning}. Usually, the shape is represented using an occupancy function on the grid. Most approaches use autoencoders or GANs as the generative model.

Multiple works take different inputs such as a 3D scan with missing parts or images: \cite{dai2016shape} use convolutional neural networks in order to fill in missing parts in scanned 3D models, a task that was also recently targeted by \cite{wang2017shape} (using GANs and recurrent convolutional networks). \cite{gadelha20163d} propose to learn a distribution of 3D shapes from an input of images of these models using a novel 2D projection layer. In a related work \cite{zhu2017rethinking} try to generate a 3D model from a single image.
Another type of input can be supplied by the user: \cite{liu2017interactive} suggested a system that is based on 3D GANs and user interaction that generates 3D models.
A main drawback of these volumetric approaches is the high computational load of working in discretized 3D space which results in low resolution shape representation. \cite{tatarchenko2017octree} tried to bypass this problem by using smart data structures (\eg, octree) for 3D data. Another disadvantage is the fact that volumetric indicators are not optimal for smooth surface approximation, resulting in brick-like shape approximation.


\paragraph{Point cloud data generation.}
Some works have targeted the generation of 3D point clouds. This type of 3D data representation resolves the resolution limitation of the volumetric representation, but introduces new challenges such as points' order invariance and equivariance \cite{qi2017pointnet,zaheer2017deep}. \cite{fan2016point} use this representation for the problem of 3D reconstruction from a single image. \cite{achlioptas2018learning} design and study autoencoders and GANs. \cite{nash2017shape} use a variational autoencoder \cite{doersch2016tutorial} in order to generate point clouds and corresponding normals.

\paragraph{Depth maps generation.}
Another group of papers have targeted the generation of depth maps (possibly with normals). A depth map is a convenient representation since it is formulated as a tensor (a regular grid of numbers) similarly to standard images. 

\cite{tatarchenko2016multi} use an encoder-decoder architecture in order to generate a depth map from a single image. \cite{wu2017marrnet} suggest an end to end framework that takes images and generates voxelized 3D models of the shape in them, by estimating depth maps, silhouettes and normal maps as an intermediate step in the network. \cite{lun20173d} use drawings as input and generate multi view depth maps and normals which are again fused together to a single output.
In a different variation, \cite{soltani2017synthesizing} learn a model that takes a depth map or a silhouette of a shape and generates multiview depth maps and corresponding silhouettes. Using these depth outputs they generate a single point cloud in a post process.

\paragraph{Surface generation.}
The last shape representation we discuss is a 3D triangular mesh. This representation includes both a point cloud and connectivity information and is the type of representation we use in this paper. 

\cite{litany2017deformable} targeted deformable shape completion using a variational autoencoder, but their framework also allows to sample random human shapes which is the main focus of our paper. Their main limitation, in comparison to our method, is their reliance on consistent input connectivity (\ie, input shapes with the same triangulation and full 1-1 correspondences) while we only rely on a sparse set of consistent landmarks. This allows us to learn from multiple different datasets consisting of diverse shapes with arbitrary triangulations.

\cite{sinha2017surfnet} use a parameterization to a regular planar domain (an image) and represent the surface using its Euclidean coordinates. \cite{Groueix18} use multilayer perceptrons (MLPs) in order to learn multiple parameterization functions directly (\ie, functions  $f:\Omega\subset\Real^2\rightarrow\Real^3$). 
These works are the most similar to ours: Similarly to \cite{sinha2017surfnet} we also use parameterizations into a planar domain; we use parameterizations of a cover of the surface to a torus as in \cite{maron2017convolutional}. In contrast to their work,  our parameterizations are conformal and we use multiple charts that cover the shape and preserve small details. We note that \cite{sinha2017surfnet} solve for dense correspondences of the input shapes as preprocess, which is a challenging problem that currently cannot always be accomplished with high accuracy.
\cite{Groueix18}, on the other hand, also use multiple parameterizations. Their method is more general than ours as they do not assume sparse correspondences between the shapes, nor assume that the input shapes are of sphere topology. The downside of their approach is that their generated shapes have considerably less details and the generated charts do not match with high accuracy. 

\paragraph{Pre-deep learning works.}
Multiple works have targeted shape synthesis in the pre-deep learning era. Some works concentrated on composing new shapes from components;  \cite{funkhouser2004modeling} suggested an interactive system where a user can assemble shapes from a segmented shape database. \cite{kalogerakis2012probabilistic} learn a generative component based model that is able to generate novel shapes from a certain class. 

Another line of works tried to learn the shape space of a certain class of shapes;  \cite{allen2003space,anguelov2005scape,loper2015smpl} have all targeted the shape of the human body. 
In contrast to our work, these works solve for dense correspondences using a specifically-tailored deformation model. We do not use a specific deformation model. Instead, we use a high dimensional generative model to learn the shape space. We further demonstrate that other classes of shapes (\eg, bones) can be learned by the exact same generative model.





\section{Method}

\subsection{Problem statement}
Given a collection of surfaces $\mM=\set{M^s}_{s=1}^m\subset \mS$ sampled from some distribution $ֿ\mu$ in $\mS$, a collection of surfaces of the same class (\eg, humans, bones),  our goal it to learn a \emph{generative model} $G:\Real^d \too \mS$ of $\mu$. By generative model we mean a random variable $G$ that samples from the distribution $\mu$. 

The surfaces in $\mM$ are represented as surface meshes, namely triplets of the form $M^s=(V^s,E^s,F^s)$, where
$V^s, E^s, F^s$ are the vertex set, the edge set, and the face set, respectively. We do not require the meshes to share connectivity nor that a complete correspondence between the meshes is known. Rather, we will assume only a sparse set of landmark correspondences $\mP^s=\set{p_{i}^s}$ is given,  $p_{i}^s\in V^s$, $i\in [n]$. In this paper we used $n=6$ (for bones) or $n=21$ (for humans), see Figure \ref{fig:structure_and_charts}(b) for visualization of $\mP^s$ (orange dots) on three human surfaces in $\mM$. For brevity, we will henceforth remove the superscript $s$ from $\mP$, $p$ and $M=(V,E,F)$. 
\subsection{Conformal toric charts}
\begin{wraptable}[15]{r}{0.30\columnwidth}
	\vspace{-0.4cm}\hspace{-14pt}
	\includegraphics[width=0.30\columnwidth]{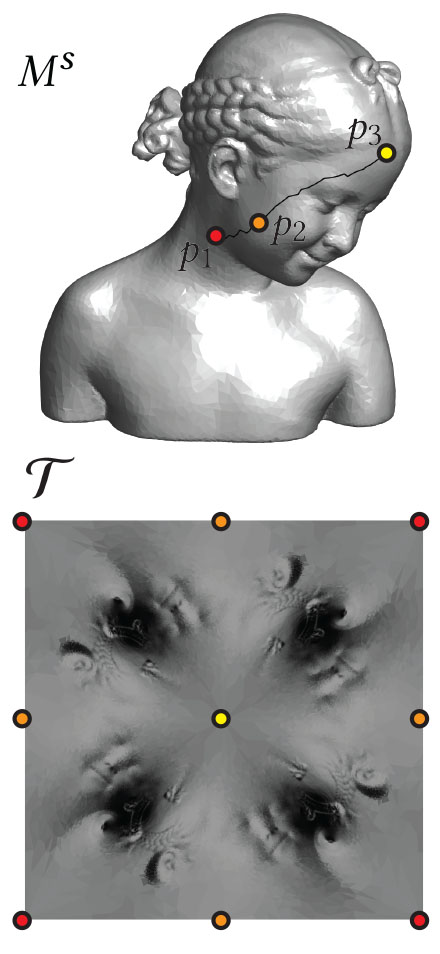}
	\vspace{-0cm}
\end{wraptable}
Our approach for learning $G$ is to reduce the surface generation problem to the image generation problem and use state of the art image GANs. The reduction to the image setting is based on a generalization of \cite{maron2017convolutional} to the multi-chart setting. \cite{maron2017convolutional} computes charts from the image domain to a surface $M$ using conformal charts, $\Phi_{P}:\mT \too M^4$, where $P=\set{p_1,p_2,p_3}\subset V$ is a triplet of points, $M^4$ is a topological torus, constructed by stitching four identical copies of $M$, and $\mT$ is the flat torus, namely the square $[-1,1]^2$ where opposite edges of the square are identified (\ie, periodic square). The torus is used as it is the only topological surface where the standard image convolution in $[-1,1]^2$, equipped with periodic padding, corresponds to a continuous, translation invariant operator over the surface. The degrees of freedom in the conformal chart $\Phi_P$ are exactly the choice of triplet $P=\set{p_1,p_2,p_3}\subset V$, where the center and corners of $[-1,1]^2$ are mapped to $p_i$, $i=1,2,3$, see the inset for an illustration.


A conformal chart, while preserving angles, can produce significant area scaling, and different choices of triplets $P$ produce low scale distortion in different areas of the surface $M$. In fact, for surfaces with perturbing parts it is impossible to choose a single triplet (chart) that provides low scale distortion everywhere. In this paper we therefore advocate a \emph{multi-chart structure} allowing to produce global, low scale distortion coverage of surfaces. 

\begin{figure*}[h]
		\includegraphics[width=\textwidth]{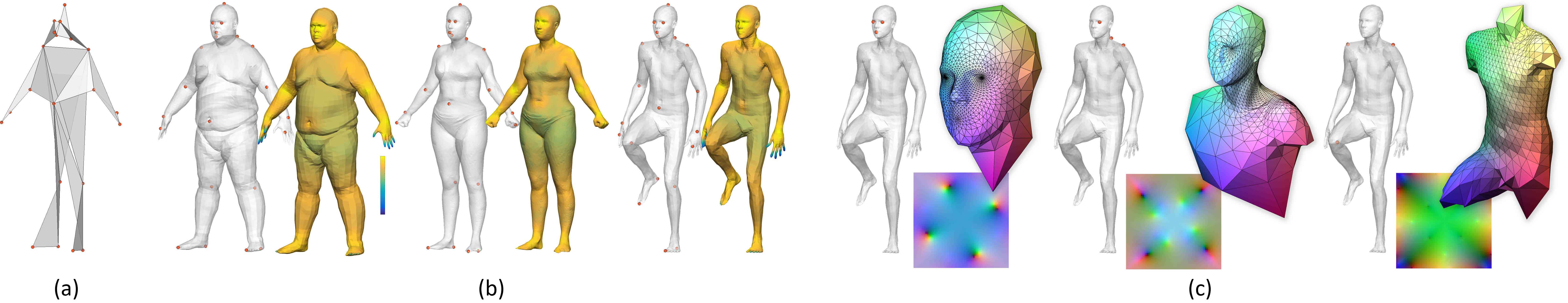}
		\caption{(a) Shows the multi-chart structure $T$; (b) shows three meshes from $\mM$, for each we show the landmark correspondences $p_{i}$ (left in each pair) and the maximal scale across all charts (right in each pair), color ranges $[10^{-1},10^3]$; (c) three charts corresponding to three different faces of $T$, we show the chart's triplet of points on the surface (left), the $(x,y,z)$ coordinates flattened to $[-1,1]^2$ (bottom) and the geometry reconstructed from this chart (right). Different charts provide low distortion coverage of different areas of the surface.}\label{fig:structure_and_charts}
	\end{figure*}

\subsection{Multi-chart structure}


 The multi-chart structure is a collection of charts that collectively represents a shape. Each chart 
 \begin{equation}\label{e:charts}
\Phi_{P}:\mT \too M^4,	
\end{equation}
is defined by a triplet of landmark points $P=(p_i,p_j,p_k)$ chosen from the collection of landmark points on the surface $\mP\subset M$.  The \emph{multi-chart structure} is therefore a pair $(\mP,T)$, where $\mP\in \Real^{n\times 3}$ is the set of landmarks and $T=(\mV, \mE, \mF)$ is an abstract triangulation, where $\mV=[n]$ is the vertex set, $\mE=\set{(ij)}$ the edge set, and $\mF=\set{(ijk)}$ the face set. Every face of the multi-chart triangulation $(ijk)\in\mF$ represents a chart $\Phi_P$, $P=(p_i,p_j,p_k)$ as in \eqref{e:charts}. We will abuse notation and write $P=(ijk)\in\mF$.
See Figure \ref{fig:structure_and_charts}(a) for a visualization of the multi-chart triangulation $T$ embedded in $\Real^3$, and \ref{fig:structure_and_charts}(b) for visualization of the landmarks $\mP$ on three human surfaces. 

Every mesh $M$ in our collection $\mM$ has $|\mF|=c$ charts (in this paper we choose $c=4$ (for bones) or $c=16$ (for humans)). Figure \ref{fig:structure_and_charts}(c) shows three charts of a single mesh $M$; for each chart we show: its defining triplet of landmarks from $\mP$ set by a face in the triangulation $P\in \mF$ (orange dots), the chart itself, $\Phi_{P}$, visualized as RGB image over $[-1,1]^2$, and the geometry captured by $\Phi_{P}$ restricted to a finite mesh overlaid on $[-1,1]^2$.
		
In order to faithfully represent shapes and enable effective 3D shape learning, the multi-chart structure should possess the following two properties: \emph{Covering property} and \emph{Scale-translation (s-t) rigidity}.

\paragraph{Covering property} Each face (triplet) in the multi-chart structure zooms-in on a different part of the surface. As the meshes are assumed to be of the same class (\eg, humans), it is usually possible to choose a multi-chart structure $(\mP,T)$ such that the chart collection $\set{\Phi_{P}}$ produces a good coverage of all meshes in $\mM$.



%

	Figure \ref{fig:structure_and_charts}(b) illustrates three different meshes colored according to the maximal area scale exerted by the different charts at each point on the surface. Note that almost everywhere the scale function is greater than $0.1$. This means that every part of the original surface is represented in at-least one of the charts with scale factor bounded by $0.1$.

		\paragraph{Scale-translation (s-t) rigidity property.} As we demonstrate in Section \ref{s:evaluation}, when training a network to predict multi-charts it is imperative that all the charts are centered (zero mean) and of the same scale (unit norm); This assures the network does not concentrate on learning large-norm charts (\eg, torso) while neglecting small-norm charts (\eg, head or hand). 
		
		Centering and scaling of the charts results in the loss of their natural scale and mean value. Thus, to reconstruct a shape from centered-scaled multi-charts (which are the output of our network) we need to recover, for each chart, a unique scale and mean (referred also as \emph{translation}). Each centered-scaled chart contains a triplet of points in $\Real^3$,		
		\begin{equation}	\label{e:r}(r_{\scriptscriptstyle{P},i},r_{\scriptscriptstyle{P},j},r_{\scriptscriptstyle{P},k})\in\Real^{3\times 3}, \qquad P=(ijk)\in \mF 
		\end{equation}
		that are a centered-scaled version of the original landmarks $(p_i,p_j,p_k)$ in $M$. S-t rigidity is the property that allows reconstructing the original scale and translation (\ie, mean) of the charts:
		\begin{definition}\label{def:rigidity}
	 	A multi-chart structure $(\mP,T)$ is \emph{scale-translation (s-t) rigid} if given a set of centered-scaled triplets, Eq.~\eqref{e:r}, the original landmarks $\mP=\set{p_i}$ can be recovered uniquely up to a global scale and translation.
		\end{definition}
				
		Let us provide an algebraic characterization to s-t rigidity. Let 
		 $r\in\Real^{3\times 3\times |\mF|}$  be the positions in $\Real^3$ of every vertex in every face (\ie, chart) $P=(ijk)\in \mF$.
		Points in $r$ corresponding to the same vertices in the triangulation might not be equal (recall that each face is centered and scaled). We would like to find translation $b_{\scriptscriptstyle{P}}\in \Real^3$ and scale $a_{\scriptscriptstyle{P}}\in \Real$ per face $P\in \mF$ to reverse the centering and scaling and obtain a unique consistent embedding $q=(q_1,\ldots,q_n)^T\in \Real^{n\times 3}$ of the vertices $\mV$, up to global scale and translation. Consistent means each vertex has the same coordinates in each triangle it belongs to. $a_P,b_P$ are solutions to the linear system:
		\begin{equation}\label{e:scale_translation}
		a_{\scriptscriptstyle{P}} r_{\scriptscriptstyle{P},l} + b_{\scriptscriptstyle{P}} = q_{l}, \quad \forall P=(ijk)\in \mF, \quad  \forall l\in \set{i,j,k}.
		\end{equation}
		This is a homogeneous over-determined system of equations where for each solution $q$, also its global scales $\alpha q_l$, $\alpha\in \Real$ and/or global translations $q_l + \beta$, $\beta\in\Real^3$ are solutions. 
		To set a unique solution we need to set the scale and translation of a single chart, $P_0\in \mF$,
		\begin{equation}\label{e:P0}
		a_{\scriptscriptstyle{P_0}} = 1, \quad b_{\scriptscriptstyle{P_0}}=0.
		\end{equation}

		%

		S-t rigidity can be equivalently stated in terms of the linear system \eqref{e:scale_translation}-\eqref{e:P0}:
		\begin{proposition}\label{prop:rigidity2}
			A multi-chart structure $(\mP,T)$ is scale-translation rigid iff the linear system \eqref{e:scale_translation}-\eqref{e:P0} has full column-rank. 
		\end{proposition}

		We prove this proposition in Appendix \ref{appendixA}. Next, we claim that s-t rigidity is a property depending only on the abstract triangulation $T$ and not on a specific choice of landmarks $\mP$.

\begin{figure*}[t]
	\includegraphics[width=\textwidth]{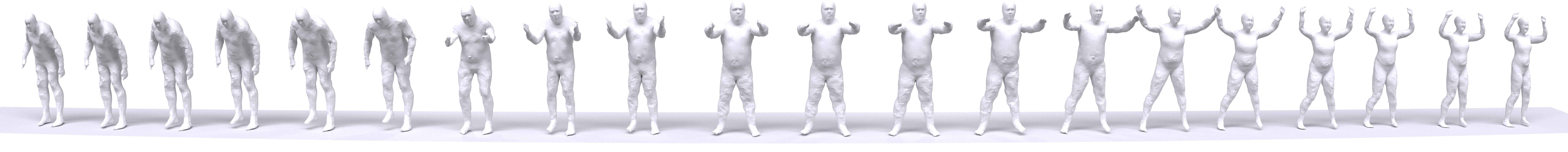}
	\caption{Equispaced interpolation between two humans of different body characteristics. } \label{fig:Humannterpolation} 
\end{figure*}
		
		\begin{theorem}\label{thm:st_generic}
			A multi-chart structure $(\mP,T)$ is either scale-translation rigid for almost all $\mP$ or not scale-translation rigid for any $\mP$.
		\end{theorem}
		This theorem can be proved using the fact that a non-zero multivariate polynomial is non-zero almost everywhere \cite{caron2005zero}. The full proof can be found in Appendix \ref{appendixA}.
		  
		It is so far not clear that s-t rigid triangulations $T$ even exist. Furthermore, Proposition \ref{prop:rigidity2} does not provide a practical way for designing multi-chart triangulations $T$ that are s-t rigid. The following theorem provides a simple sufficient condition for s-t rigidity. The condition is formulated solely in terms of the connectivity of $T$, and apply to all \emph{generic} $\mP$, that is $\mP$ where every 4 landmarks are not co-planar.
		\begin{theorem}\label{thm:rigidity}
			A 2-connected triangulation $T$ with chordless cycles of length at most 4 is scale-translation rigid. 
		\end{theorem}
		Chordless cycles are cycles in the graph that cannot be shortened by an existing edge between non-consecutive vertices in the cycle. The theorem is proved in Appendix \ref{appendixA}. The inset shows three triangulations (from left to right): an s-t non-rigid $T$ due to failure of the 2-connectedness (at the yellow vertex, for example); s-t rigid $T$; and s-t non-rigid $T$ with a chordless cycle of length 5. 
		\begin{wraptable}[3]{r}{0.30\columnwidth}
	\vspace{-10.pt}\hspace{-14pt}
	\includegraphics[width=0.30\columnwidth]{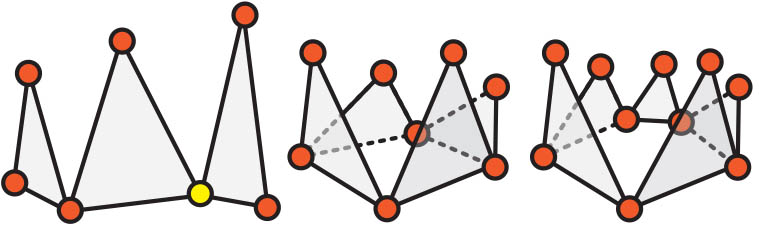}
	\vspace{-0cm}
\end{wraptable}		
		Several comments are in order: First, as shown in the inset there are graphs with chordless cycle of length 5 that are not s-t rigid and therefore the above theorem cannot be strengthened by simply replacing 4 with 5; second, the theorem can be strengthened by considering only cycles between s-t rigid components; third, the generic condition can be weakened by enforcing it only on chordless cycles. Lastly, the notion of s-t rigidity is related to the notion of parallel rigidity. A graph $G$ is parallel rigid if it does not have a non-trivial parallel redrawing, where parallel redrawing is a different graph $G'$ with edges parallel to the edges in $G$.  Necessary and sufficient conditions for parallel rigidity exist (see \eg, Theorem 8.2.2 in \cite{whiteley1996some}), however these conditions are harder to work with in comparison to Theorem \ref{thm:rigidity}.
		
		In this paper we use multi-chart structures $(\mP,T)$ with triangulations $T$ that satisfy the sufficient condition to s-t rigidity as described in Theorem \ref{thm:rigidity}. Figures \ref{fig:structure_and_charts}(a)-(b) show this multi-chart structure in the case of human body shape.

\subsection{Mesh to tensor data}
The multi-chart structure $(\mP,T)$ is used to transfer the input mesh collection $\mM$ into a collection of standard image tensor data as follows. 

We consider the coordinate functions over the meshes, $X=(x,y,z):M\too\Real^3$, and use our multi-chart structure $(\mP,T)$ to transfer these coordinate functions to images. Given a chart $P\in \mF$, we pull the coordinate functions to the flat torus $\mT$ via
\begin{equation}
	X_{P} = X\circ\Phi_{P}
\end{equation}
and sample it on a regular $k\times k$ grid of $[-1,1]^2$, where in this paper we use $k=65$. This leads to tensor input data $Y_{P} \in \Real^{k\times k \times 3}$. Figure \ref{fig:structure_and_charts}(c) shows three tensors $Y_{P}$ as colored square images for three different charts $P\in \mF$.  Concatenating all charts per mesh gives the final multi-chart tensor representation of mesh $M$,
\begin{equation}\label{e:Ys}
	Y\in\Real^{k\times k \times 3|\mF|}.
\end{equation}
$Y$ contains all geometric data for mesh $M$, and the entire input tensor data is $\set{Y^s}_{s=1}^m$.
Differently from images that contain $3$ channels, every instance of our data $Y$ contains $3|\mF|$ channels in $|\mF|$ groups. Each contains the three coordinate functions of the surface transferred to $\mT$ using a different chart. As discussed above,  since different charts have different mean and scale/variance (\eg, torso and head) it is important that the different channels in $Y$ are normalized, \ie, each $Y_{P}$ is centered and scaled to be of unit norm (variance). Otherwise the learning process is suboptimal for the small, non-centered charts, \eg, those corresponding to head and hands. Therefore, in our data we normalize all charts, $Y_{P}$. Of course, we can only do that if there is a unique way to recover scale and translation per chart, which is the case if the multi-chart triangulation is scale-translation rigid.

The charts, $\Phi_{P}$, map $[-1,1]^2$ to four copies of the surface $M$.
%
%
Accordingly, the tensor $Y$ also contains four copies of the surface's coordinate data. 
We denote by
\begin{equation}\label{e:y_Psl}
y_{\scriptscriptstyle{P},l}\in \Real^3, \quad l \in \set{i,j,k}
\end{equation}
the entries of $Y_{\scriptscriptstyle{P}}$ corresponding to (one copy of) a triplet of landmarks $P$ in $[-1,1]^2$.


%

\begin{figure}[t]
 \includegraphics[width=\columnwidth]{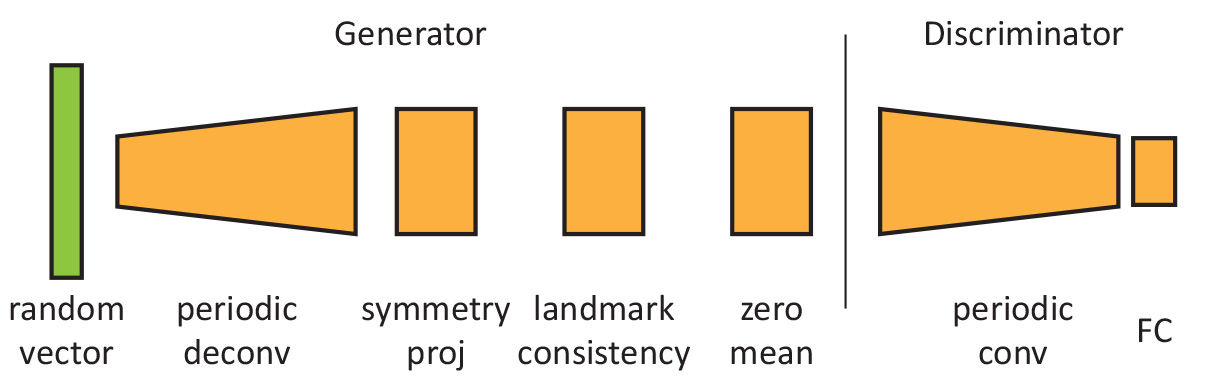}
  \caption{The generator and discriminator architecture. }\label{fig:generator}
\end{figure}

\subsection{Architecture and layers}\label{ss:arch} We apply the image GAN technique \cite{goodfellow2014generative} to learn our surface generator $G:\Real^d\too \Real^{k\times k \times 3|\mF|}$ from the input set of surfaces represented as tensors of the same dimensions, $\set{Y^s}_{s=1}^m\subset \Real^{k\times k \times 3|\mF|}$. The loss function used to train the generator is defined using a discriminator $D:\Real^{k\times k\times 3|\mF|}\too \Real$, which is also a deep network aiming to classify input multi-chart data $Y\in \Real^{k\times k \times 3|\mF|}$ to either a real surface, or a generated surface. The discriminator is fed with both real instances $Y^s$ and generated instances $Q=G(z)$ and optimizes a loss trying to correctly discriminate between the two.

In this paper we use a similar architecture to \cite{karras2017progressive} without the progressive growing part, that is, we do not change the resolution during learning. The loss we use is the Wasserstein loss \cite{gulrajani2017improved} .
 We apply the following changes to the network to adapt to our geometric setting. The architectures of the generator and discriminator are shown in Figure \ref{fig:generator} and more details can be found in Appendix \ref{appendixB}.

\paragraph{Number of channels.}
First, we change the number of output channels generated by $G$ to $k\times k \times 3|\mF|$ and rescale number of channels accordingly in both G and D, see Appendix \ref{appendixB} for all channel sizes.

\paragraph{Periodic convolutions and deconvolutions, symmetric projection.}
Second, similarly to \cite{maron2017convolutional}, all convolutions are applied with periodic padding to account for the original surface's topology. Deconvolutions are implemented by bilinear upsampling followed with a periodic convolution (as in \cite{karras2017progressive}). Furthermore, since we are working with four copies of the surface we also incorporate the (max) \emph{symmetry projection} layer after the convolution layers of the generator that makes sure all four copies are identical (again, as in \cite{maron2017convolutional}), see Figure \ref{fig:generator}.

\paragraph{Landmark consistency.}
Third, our data $Y^s$ is per-chart normalized and hence the generator will also learn (approximately) normalized charts $Q=G(z)\in\Real^{k\times k\times 3|\mF|}$, $z\in \Real^d$. One property that always holds for the data $Y$ is that there exists a unique scale and translation per chart $P$ that solves \eqref{e:scale_translation}-\eqref{e:P0} \emph{exactly}. We will therefore enforce this exactness condition on the generator output $Q$.


We implement a layer, called \emph{landmark consistency}, that given a generated tensor $Q\in \Real^{k\times k \times 3|\mF|}$ extracts the landmark values $y\in \Real^{3\times 3 \times |\mF|}$ (as in \eqref{e:y_Psl}), and solves the linear system \eqref{e:scale_translation}-\eqref{e:P0} with $r=y$, in the \emph{least-squares sense}. Note that Theorem \ref{thm:st_generic} implies the existence of a unique solution in this case, almost always. Then, we transform each triplet $y_{\scriptscriptstyle{P}}$ in $y$ by the respective scale and translation, obtaining new locations for the landmarks denoted $\hat{y}$, replacing each landmark value $\hat{y}_{\scriptscriptstyle{P},l}$ with the average of all values corresponding to the same landmark, and transforming back by subtracting the translation and dividing by the scale, $\wt{y}$. Lastly we replace the values $y$ in $Q$ with the new, consistent values $\wt{y}$.

\paragraph{Zero mean.}
Lastly, a \emph{zero-mean} layer is implemented, reducing the mean of every chart in the generated tensor $Q$. As mentioned above, this condition is also satisfied by our train data $\set{Y^s}$.

\begin{figure}[t]
	\includegraphics[width=\columnwidth]{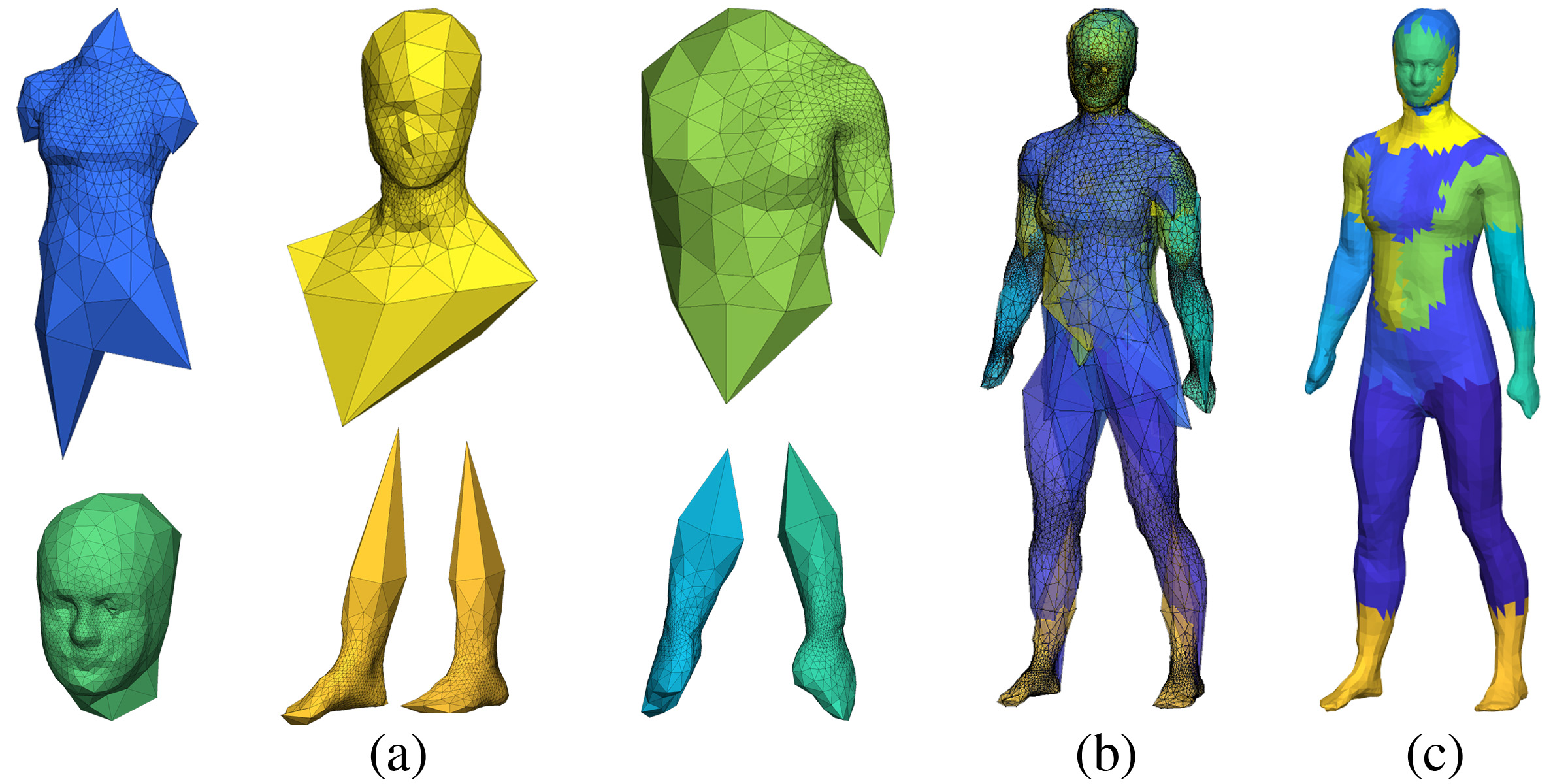}
	\caption{Reconstruction of a full shape from multiple charts. (a) several generated charts. (b) All charts after solving for scales and translations. (c) Reconstructed mesh with color coding of the charts with maximal scale used for each point.} \label{fig:reconstruction}
\end{figure}

\subsection{Reconstruction}
The last part in our pipeline is reconstructing a surface $M$ from the generative model output $Q=G(z)\in \Real^{k\times k\times 3|\mF|}$. The reconstruction includes two steps: (i) recover a scale and translation per chart in $Q$; and (ii) extract vertex coordinates of a template mesh $M^{t}$ from $Q$.

\paragraph{Recover scale and translations.}
The first step in reconstructing a surface $M$ out of the generator output $Q$ is to recover a scale $a_{\scriptscriptstyle{P}}$ and translation $b_{\scriptscriptstyle{P}}$ per chart $Q_P\in \Real^{k\times k \times 3}$. This is done by solving the linear system \eqref{e:scale_translation}-\eqref{e:P0} where $r=y$ are the landmark values from the different charts in $Q$. Since our network includes a landmark consistency projection layer (see Subsection \ref{ss:arch}) there exists an exact solution to this system. The solution to this system is unique due to the scale-translation rigidity of $T$. Let $\wh{Q}$ denote the scaled and translated charts of $Q$ by the scales and translation achieved by solving the linear system. Figure \ref{fig:reconstruction} (a) shows examples of the different charts in $Q$; and (b) shows the different charts of $\wh{Q}$ embedded in $\Real^3$ after solving for and rectifying the scales and translations.

\paragraph{Template fitting.}
In the second stage of the reconstruction process we use as template mesh, $M^t=(V^t,E^t,F^t)$, a per-vertex average of the rest-pose models in DFAUST \cite{bogo2017dynamic}.  We reconstruct the final mesh $M=(V,E,F)$ using data from $\wh{Q}$. We use the connectivity of $M^t$ (\ie, $E^t$ and $F^t$) and set the vertices location using the multi-chart structure $(\mP,T)$ as follows,
\begin{equation}\label{e:reconstruction}
	v = \frac{\sum_{P\in \mF} \tau_{\scriptscriptstyle{P}}(v)\, \wh{Q}(\Phi^{-1}_{P}(v))}{\sum_{P\in \mF} \tau_{\scriptscriptstyle{P}}(v)}
\end{equation}
where $\tau_{\scriptscriptstyle{P}}(v)$ is the inverse area scale of the $1$-ring of vertex $v$ exerted by chart $\Phi_{P}$ of the template mesh, $M^t$; $\wh{Q}(\Phi^{-1}_{P}(v))$ is the image of the point $u=\Phi^{-1}_{P}(v)$ under the learned chart $\wh{Q}$, computed by bilinear interpolation of $\wh{Q}$ in each of its grid cells. Equation \ref{e:reconstruction} makes sense since each point's coordinate is mainly influenced by the charts that represent it well. Figure \ref{fig:reconstruction}(c) shows the final reconstruction $M$ with color coding of the charts with maximal scale used for each point; note the similarity of the (b) and (c).



\section{Implementation Details}

\subsection{Datasets}
\paragraph{Humans} The training set we have used for human body generation consists of two large datasets of human models: DFAUST \cite{bogo2017dynamic} and CAESAR \cite{yang2014semantic}. The DFAUST dataset contains $40k$ models in multiple body poses of ten different people. The CAESAR  dataset compliments DFAUST and contains about $3k$ models of different people in rest pose. Both datasets are aligned internally. We align both datasets by removing the mean from each shape, scaling it to have a surface area of $1$ and solving for the optimal rotation to fit a set of landmarks between the datasets using Singular-value decomposition (SVD). As each of these datasets has consistent vertex numbering, we manually select the set of landmarks $\mP$ on a single model from each dataset. 
For this shape class we used a multi-chart structure that consists of 16 triangles and 21 landmarks. This is demonstrated in figure \ref{fig:structure_and_charts}.
In order to make our training set balanced we chose 8244 models from DFAUST (by taking each fifth shape) and doubled the number of CAESAR models to 5750.

\paragraph{Bones}  We also evaluated our method on anatomical surfaces \cite{boyer2011algorithms}. We used 70 models and as in \cite{maron2017convolutional} we converted the meshes to sphere-type topology. We also extrinsically aligned the teeth using their landmarks. For this shape class we used a multi-chart structure that consists of 4 triangles and 6 landmarks.

\begin{figure}[t]
	\includegraphics[width=\columnwidth]{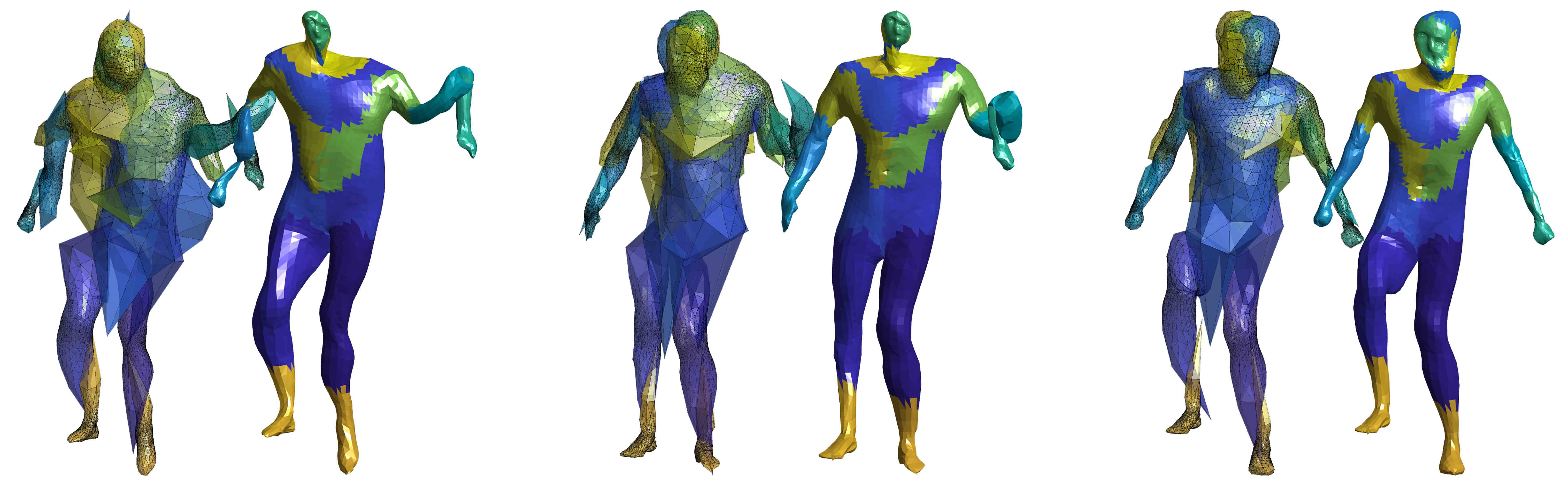}
	\caption{A naive single-chart surface generation approach with \cite{maron2017convolutional}. In each pair: left - generated charts; right - reconstructed surface. Although generating individual charts of high quality, the different charts are learned independently and consequently do not fit. } 
	\label{fig:maronComparison}
\end{figure}

\subsection{Training details}
We implemented the networks using the TensorFlow library \cite{tensorflow2015-whitepaper} in python. For the larger network that generates human models, we perform synchronous training on 2 NVIDIA p100 GPUs and for the smaller network that generates teeth we use a single p100 GPU. During training, we alternate between processing a batch for the generator and processing a batch for the discriminator. One epoch takes $\sim240$ sec and $\sim0.25$ sec for the humans and teeth networks respectively. The networks converge after 800  and 80k epochs for the humans and teeth respectively. Generating a new surface takes $1.03$ sec, from which the feed-forward takes $0.03$ sec on a single p100 GPU and the reconstruction takes $1$ sec (CPU).

Due to noise in $Q$ during the learning process we start the training with no landmark consistency layer. After 50 or 10k epochs for humans and teeth, respectively, we add the landmark consistency layer. To avoid bias in scale and translation we randomize the fixed chart $P_0$ at each iteration. Furthermore, to overcome numerical instabilities we add a regularization term to the least-squares \eqref{e:scale_translation}-\eqref{e:P0}  system of the form
\begin{equation}\label{e:regularize}
\lambda \sum_{P\in \mF} (a_{\scriptscriptstyle{P}}-\bar{a}_{\scriptscriptstyle{P}})^2,
\end{equation}
where $\lambda$ is a parameter and $\bar{a}_P$ is the average scale of the $P^{\mathrm{th}}$ chart as computed in a preprocess across the entire data $\set{Y^s}$. We set $\lambda=10$ for the next 450 epochs and then reduced $\lambda$ by a multiplicative factor of $0.995$ every epoch, until a total of 800 epochs is reached. For the teeth generating network the addition of the regularization term was not needed.

\section{Evaluation}\label{s:evaluation}

\begin{figure}[t]
		\begin{tabular}{c}
		\includegraphics[width=\columnwidth]{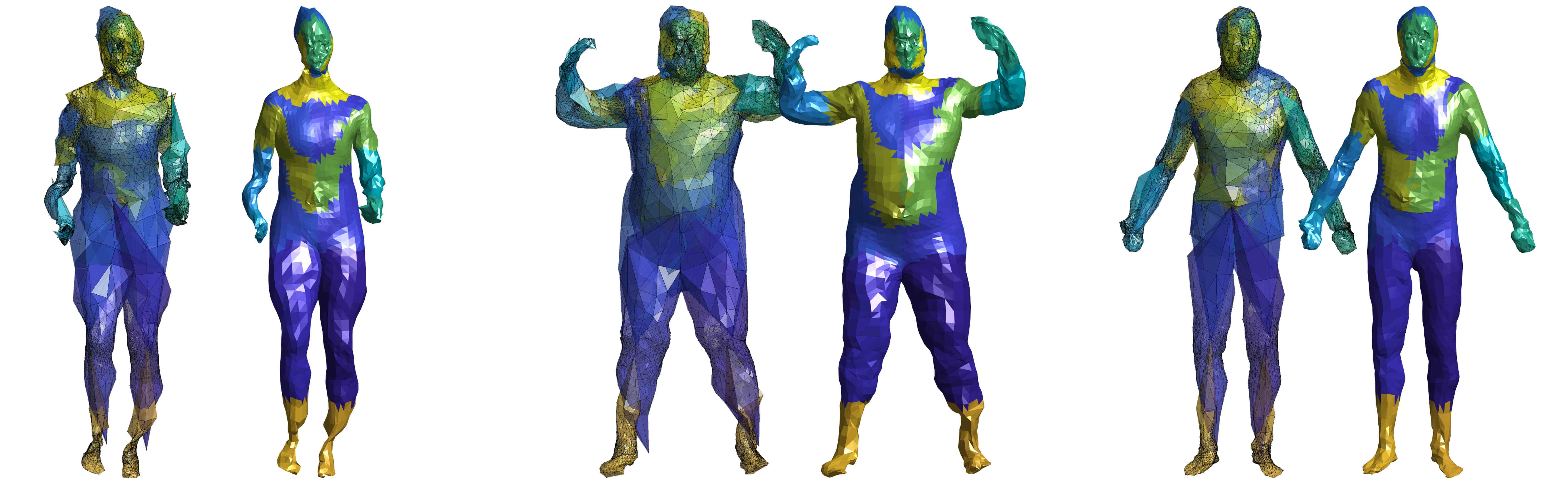}		\\
		No normalization \\ 
		\includegraphics[width=\columnwidth]{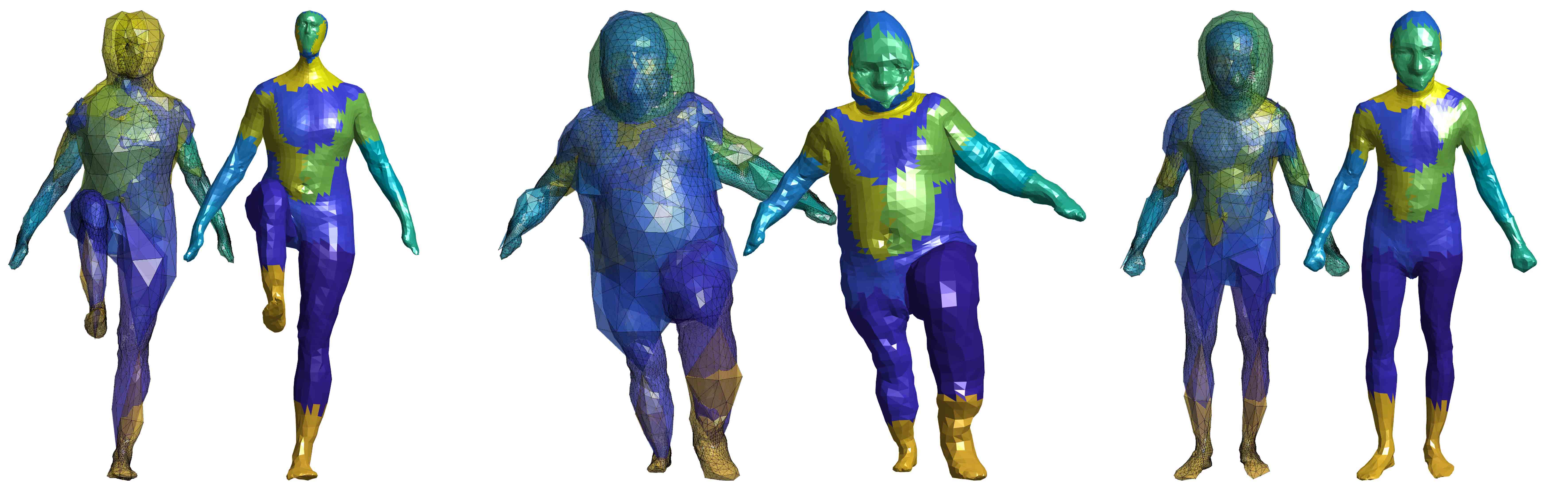} 		\\
		No projection \\
		\includegraphics[width=\columnwidth]{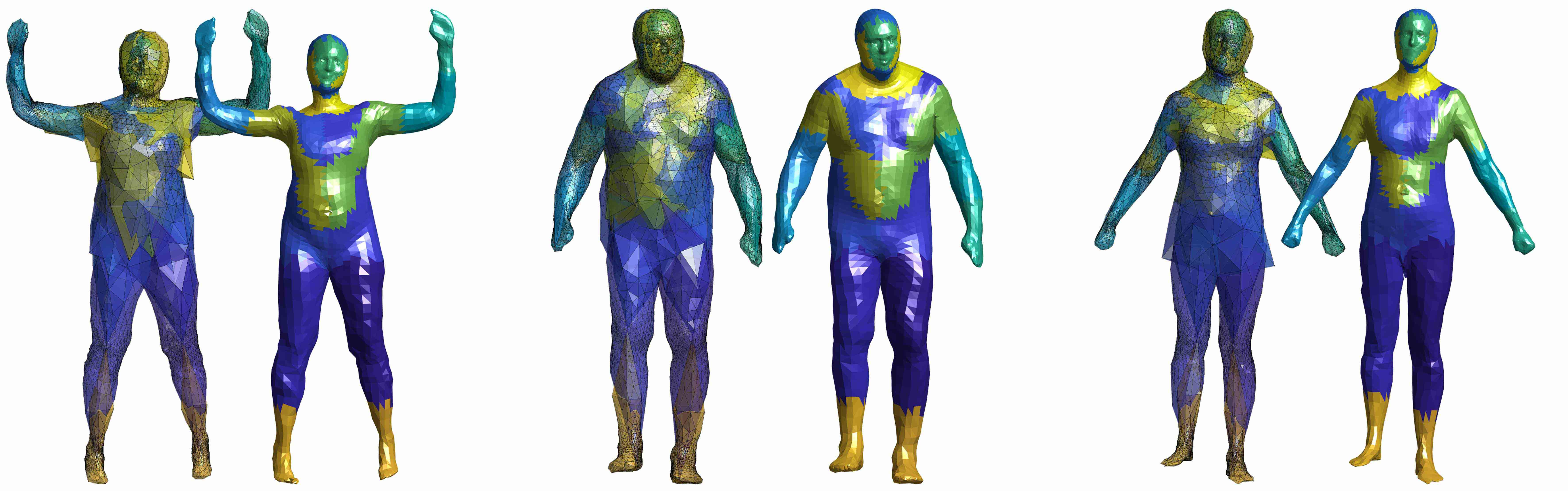}\\
		Ours \\
	\end{tabular}
	\caption{Comparison with two variations of our algorithm. In each pair, the left model shows all the individual generated charts and the right model is a final reconstruction.} 
	\label{fig:methodEval}
\end{figure}

\begin{figure*}[t]
	\includegraphics[width=\textwidth]{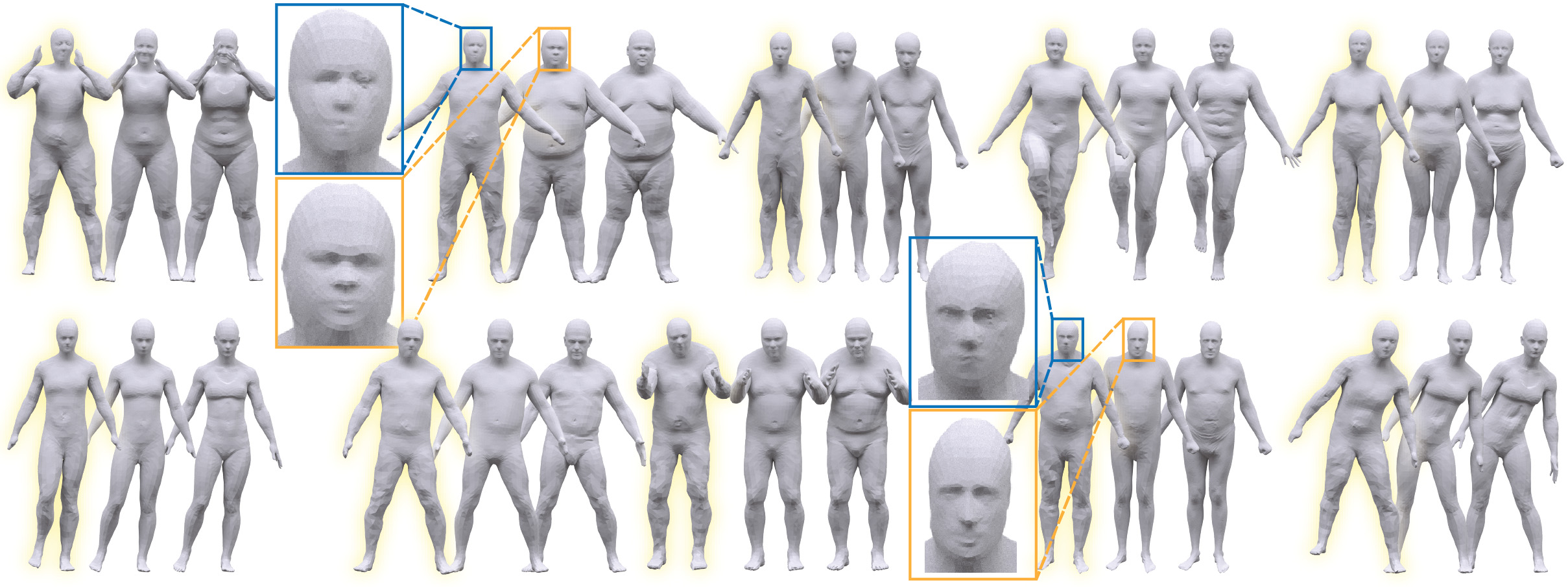}
	\caption{Comparison of human models generated by our method (left in each triplet) and their nearest neighbors in the training set (middle and right in each triplet). In the middle of each triplet we show the nearest training model reconstructed from its charts $Y$ using our reconstruction pipeline; on the right we show the original surface mesh from the dataset. The blow-ups emphasis the differences between generated and real face examples. } \label{fig:comparisonVSNNFullBody}
\end{figure*}

In this section we compare our method to several baseline methods, all of which are variations of our approach. We also present a nearest neighbor evaluation, for testing the ability of our method to generate novel shapes.

\subsection{Single-chart surface generation}
A naive adaptation of the approach presented in \cite{maron2017convolutional} to surface generation is to train a network that generates a single chart at each feed-forward and stitch the generated charts in a postprocess. The output of G of this network is a single chart of dimensions $\Real^{k\times k\times 3}$ and the capacity of the network was reduced compared to the multi-chart network accordingly. We trained this network using the same data we used for our method, feeding a random chart at each iteration.  
Figure \ref{fig:maronComparison} shows a few typical examples generated using this approach. In order to generate the first two models (left and middle) we selected random charts until we had all 16 necessary charts. For the last model we cherry-picked specific charts that seemed to fit reasonably. In all cases we ran our reconstruction algorithm, with the exception of using the mean charts' scales and solving only for the translations (solving for the scales as well resulted in worse results). This comparison shows that different charts of the same shape should be jointly learned.

\subsection{Chart normalization and landmark consistency}
We compared our method to two other baseline methods:  (a) Learning the multi-chart structure without chart normalization (centering and scale), (b) Learning the multi-chart structure without the landmark consistency layer (as described in \ref{ss:arch}).
Figure \ref{fig:methodEval} compares baselines (a)-(b) to our method by depicting several typical examples. The first row shows baseline (a), the second row baseline (b) and the third row our algorithm (with normalization and landmark consistency). Note that both the normalization step and the landmark consistency layer are important in order to generate smooth and consistent results.

\subsection{Nearest neighbor evaluation} 
In order to test the method's ability to generate unseen shapes, we apply our trained generator $G$ to multiple random latent variables $G(z)$, $z\in\Real^d$ and compare the resulting charts to their nearest neighbor in the training data $\set{Y^s}_{s=1}^m$ using $L_2$ norm in $\Real^{k\times k \times 3|\mF|}$. In the experiment, shown in Figure \ref{fig:comparisonVSNNFullBody}, we show: left, the reconstructed $M$ from the generated example $G(z)$; middle, the closest model $M^s\in \mM$ in the training set reconstructed from its charts $Y^s$ using our reconstruction pipeline; right, the closest model $M^s$ in its original surface form.

%

\section{Results}

 \begin{figure*}[!ht]	
 	\begin{tabular}{c}
\includegraphics[width=\textwidth]{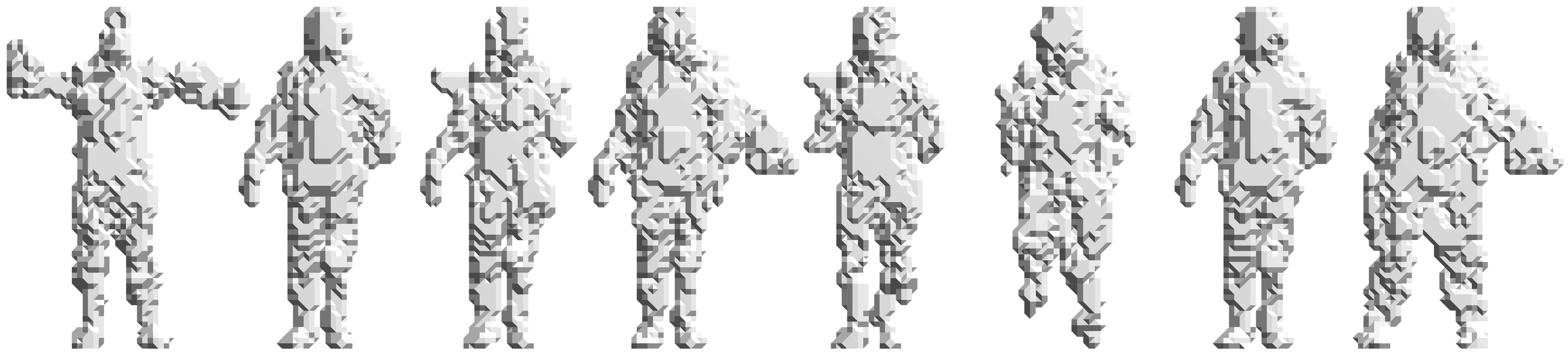}		\\
 		Volumetric GAN \\  
\\
 		\vspace{10pt}
\includegraphics[width=\textwidth]{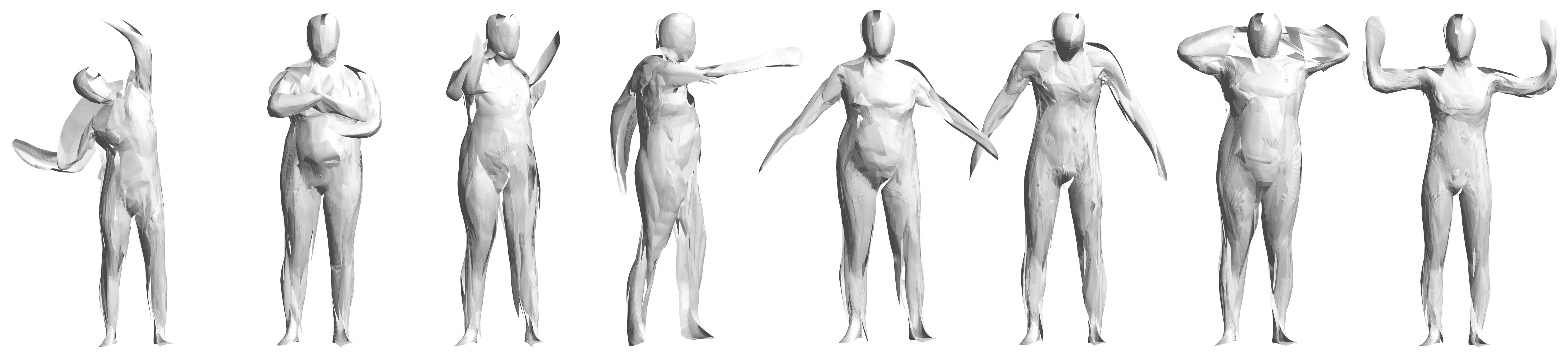} 		\\
 		AtlasNet \cite{Groueix18} \\ \\
 		
 		 		\vspace{10pt}
 		\includegraphics[width=\textwidth]{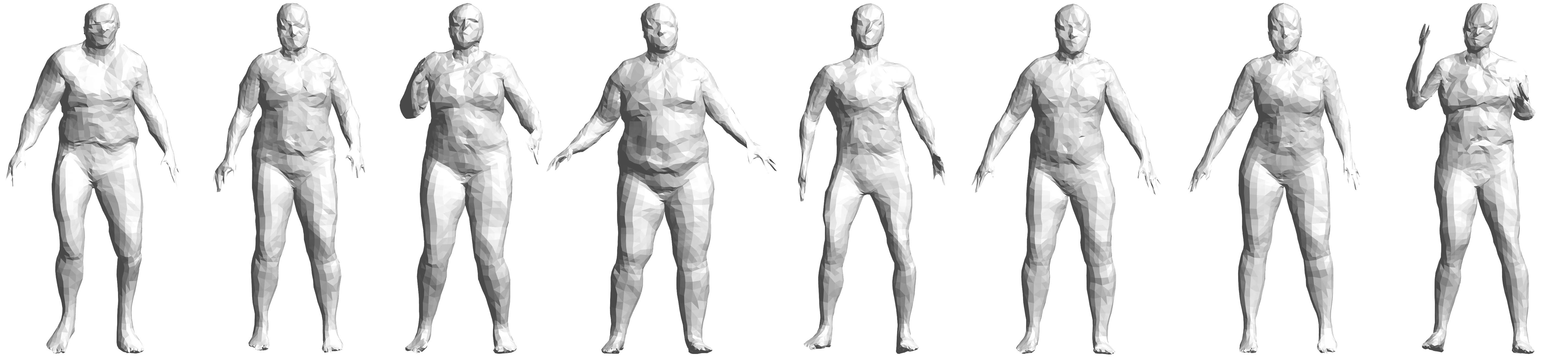} 		\\
 		Litany \etal  ~\cite{litany2017deformable} \\ \\		
 		\vspace{10pt}
 		
\includegraphics[width=\textwidth]{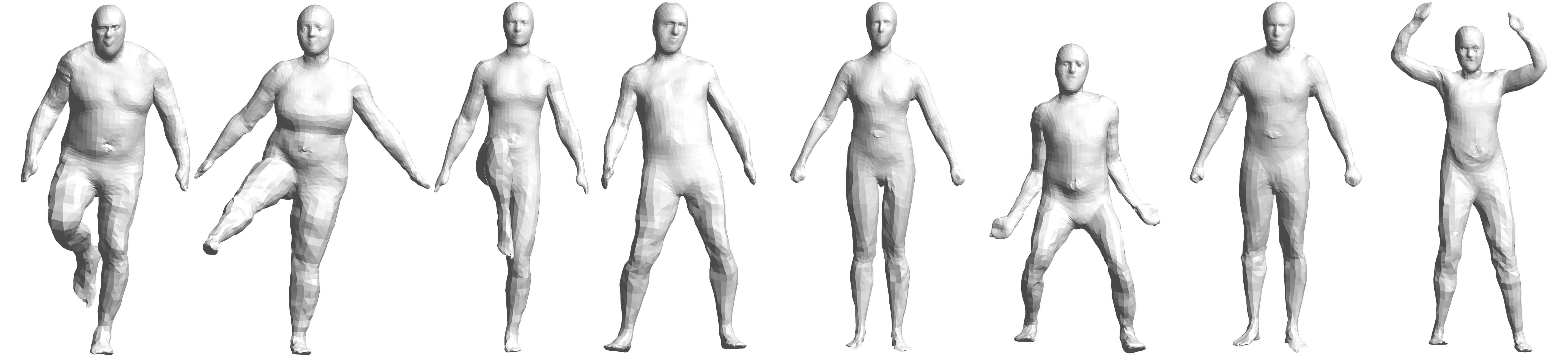}\\
 		Ours  \\	\\ 	
 		
 	\end{tabular}	
	\caption{Human shape generation. Comparison of our method with volumetric GAN baseline, the approach of \protect{\cite{Groueix18}} and of \cite{litany2017deformable}. }
\label{fig:comparisonToOtherMethods} 
\end{figure*}

\subsection{Comparison with alternative approaches}
We compare our method with a volumetric GAN approach, a recent approach by \cite{Groueix18} and the approach of \cite{litany2017deformable}. Figure \ref{fig:comparisonToOtherMethods} shows 8 results generated with each approach.

The volumetric method is implemented according to \cite{wu2016learning} with $64^3$ resolution (comparable to our $64\times 64\times 48$ tensors). The volumetric generator tends to produce crude, brick-like approximations of the surface shapes, hindering representation of specific body details.
The results of \cite{Groueix18} were provided by the authors and were trained only on FAUST (200) models \cite{bogo2014faust}. Although this is a smaller dataset than the one we used, the differences between the level of details and surface fidelity are clear.  
The results of \cite{litany2017deformable} were provided by the authors and are obtained by training on the DFAUST dataset \cite{bogo2017dynamic}. Note that their variational autoencoder was trained for a different task - shape completion. For this task they have explicitly relaxed the gaussian prior during training which (as mentioned by the authors) might give rise to generation of  slightly unrealistic shapes. 

\begin{figure*}[h]
	\includegraphics[width=\textwidth]{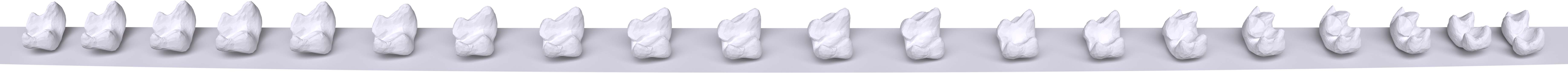}
	\caption{Equispaced  interpolation between two teeth  models. } \label{fig:teethIterpolation} 
\end{figure*}

\subsection{Shape interpolation}
Our method learns a map from the latent variable space $\Real^d$ to shape space $\Real^{k\times k \times 3|\mF|}$. This gives us the ability to perform interpolation between two generated shapes $G(z_1),G(z_2)$. Figure \ref{fig:Humannterpolation} shows equispaced samplings of a latent space line segment $[z_1,z_2]$ between two humans in different poses and body characteristics. Note how the models change in a continuous manner through other, natural models and poses. Figure \ref{fig:teethIterpolation} shows a similar experiment with the teeth surface dataset. The supplementary movie shows interpolation between different humans (and teeth) in the latent space.

\subsection{Shape exploration}
In this experiment, shown in Figures \ref{fig:gridSamplingHuman}, \ref{fig:gridSamplingTeeth}, we computed a 2D grid using bilinear interpolation on four latent vectors $z_1,z_2,z_3,z_4\in\Real^d$ and generated the corresponding models. Note how the grid captures gracefully the pose space. These types of grids can be used as means to browse datasets and shape spaces.

Failure cases: Figure \ref{fig:fail} shows the result of an experiment of 100 random models generated by our method, where failures are marked in red. Note that the ratio of failures is less than $5\%$, and in general the failures are also rather plausible human shapes. 
 
\begin{figure*}[h]
	\includegraphics[width=\textwidth]{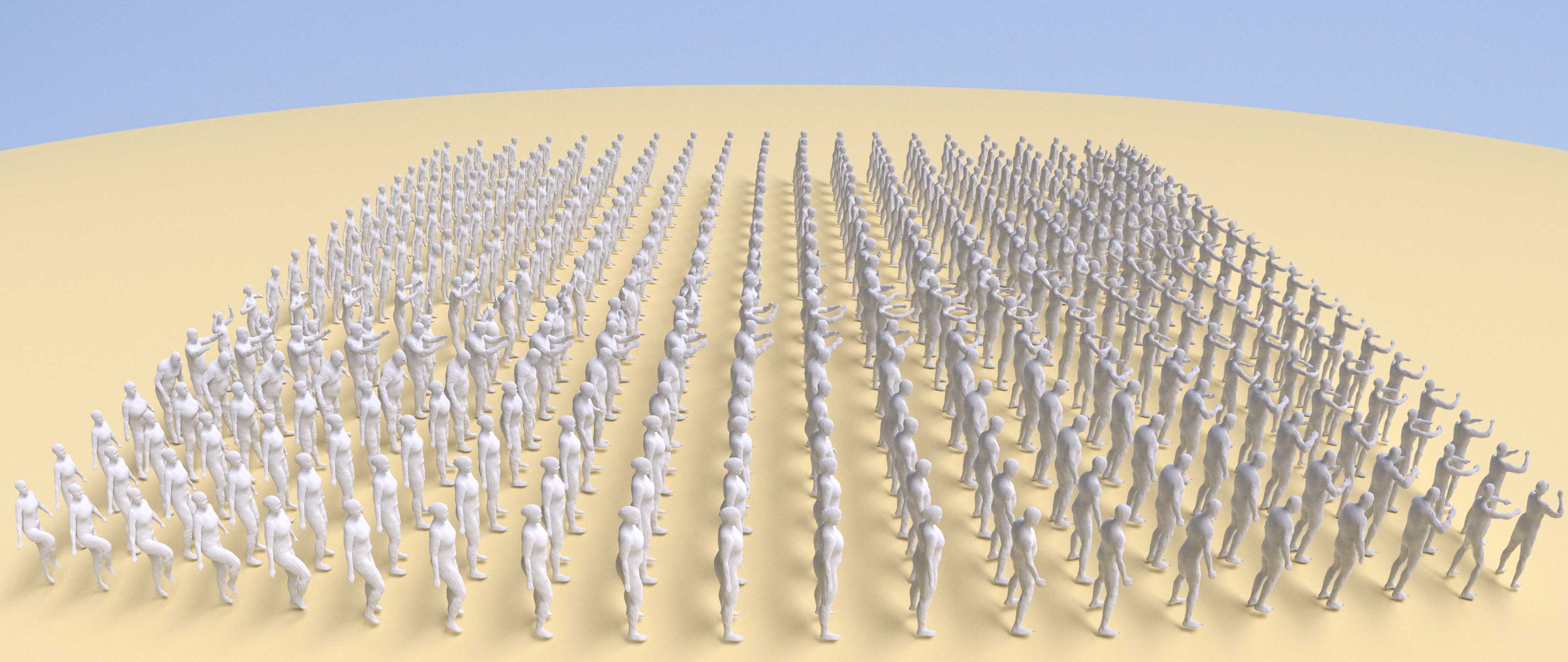}
	\caption{Bilinear interpolation of four generated human models. \vspace{10pt}} \label{fig:gridSamplingHuman} 
	\vspace{-20pt}
\end{figure*}
\begin{figure*}[h]
	\includegraphics[width=\textwidth]{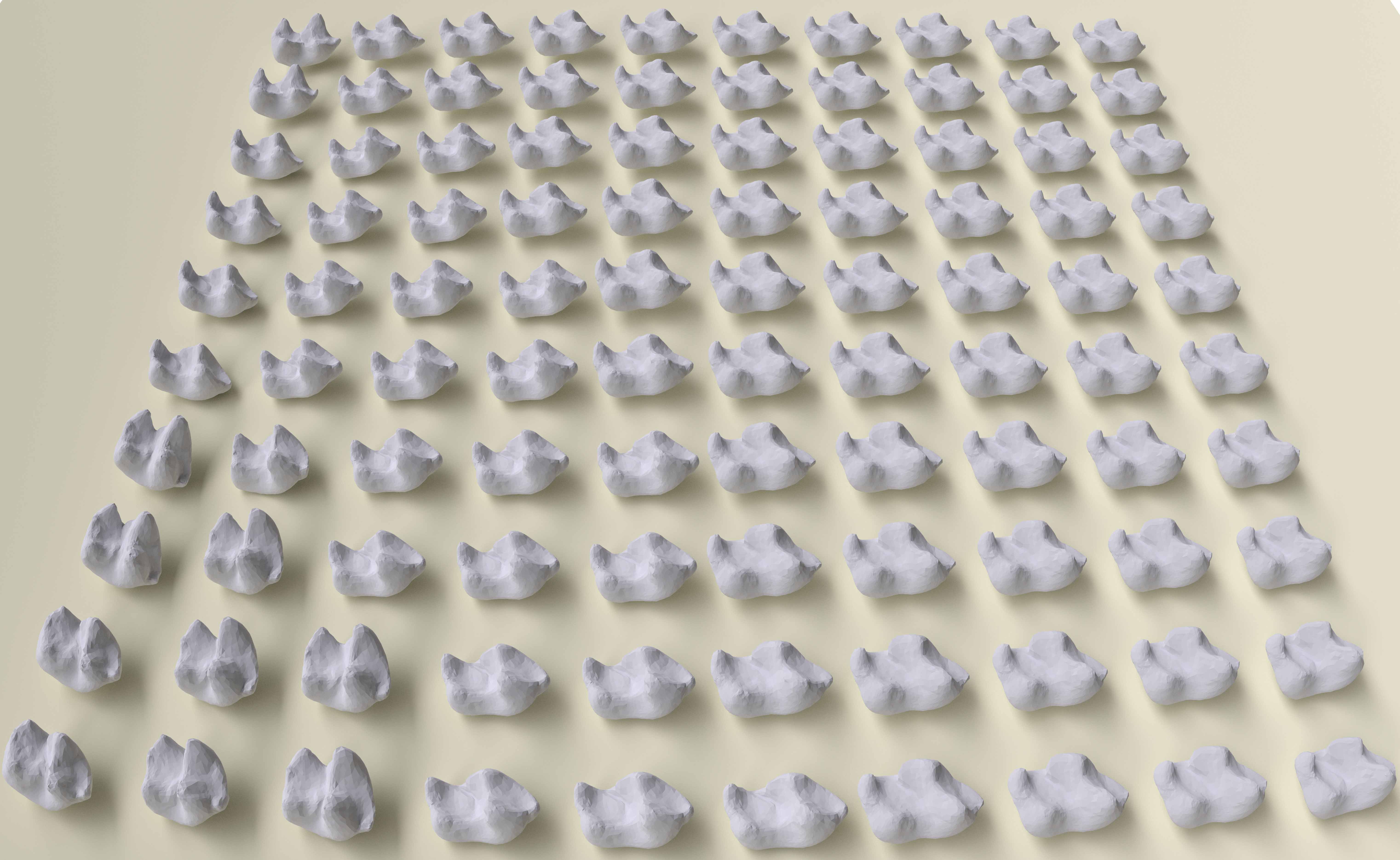}
	\caption{Bilinear interpolation of four generated teeth models. \vspace{10pt}} \label{fig:gridSamplingTeeth} 
	\vspace{-20pt}
\end{figure*}

\begin{figure*}
	\includegraphics[width=0.95\textwidth]{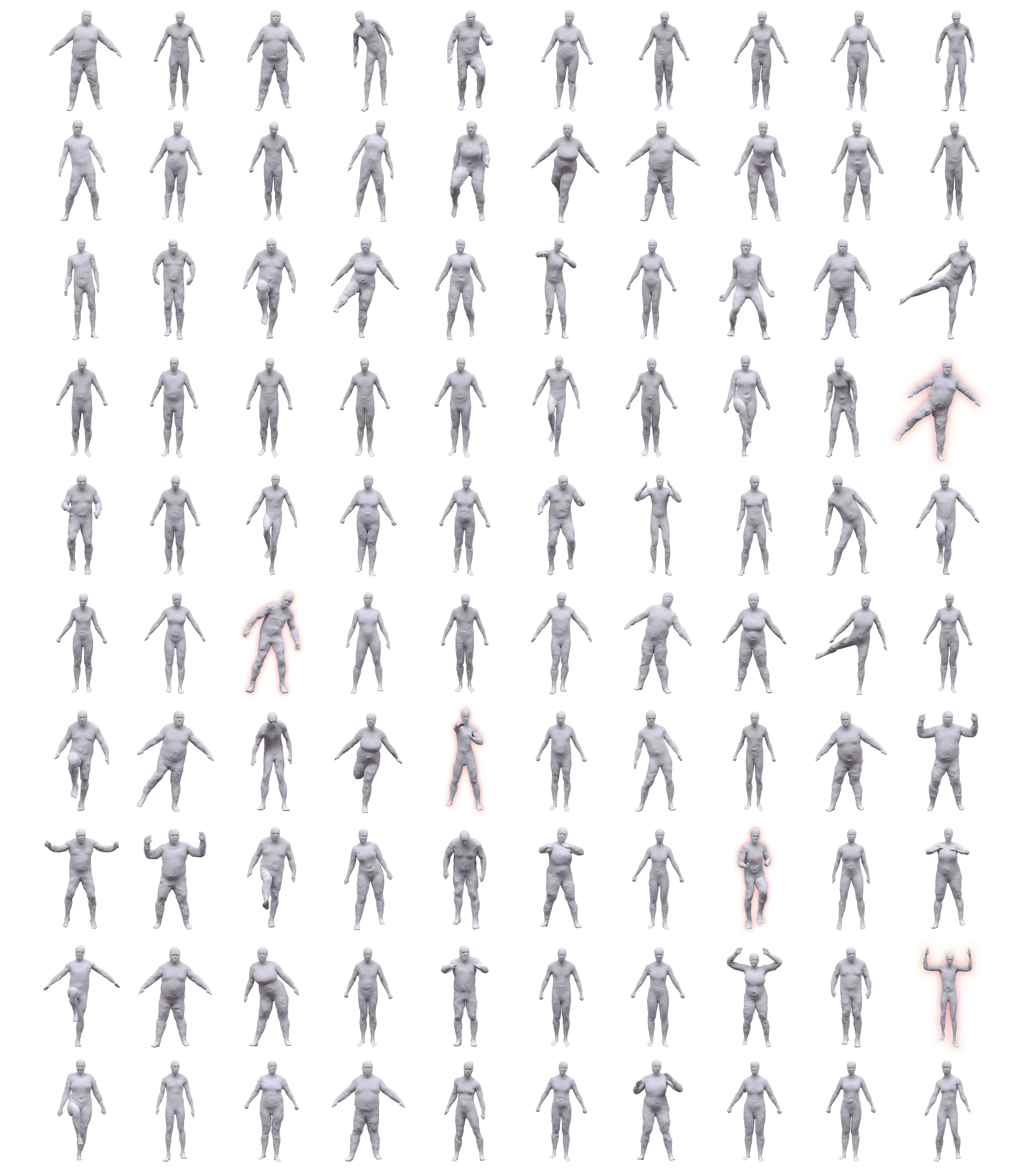}	
	\caption{100 random human shape generation with our method. The failures are shown in red. }\label{fig:fail}
\end{figure*}
\vspace{10pt}

\subsection{Massive-scale data generation.}  Lastly, our method can be used for massive generation of plausible random models. Figure \ref{fig:massive} shows $10k$ human models generated by our method, completely automatically. Note the diverse poses and different faces our method is able to generate without human intervention.

\begin{figure*}[h]
	\includegraphics[width=\textwidth]{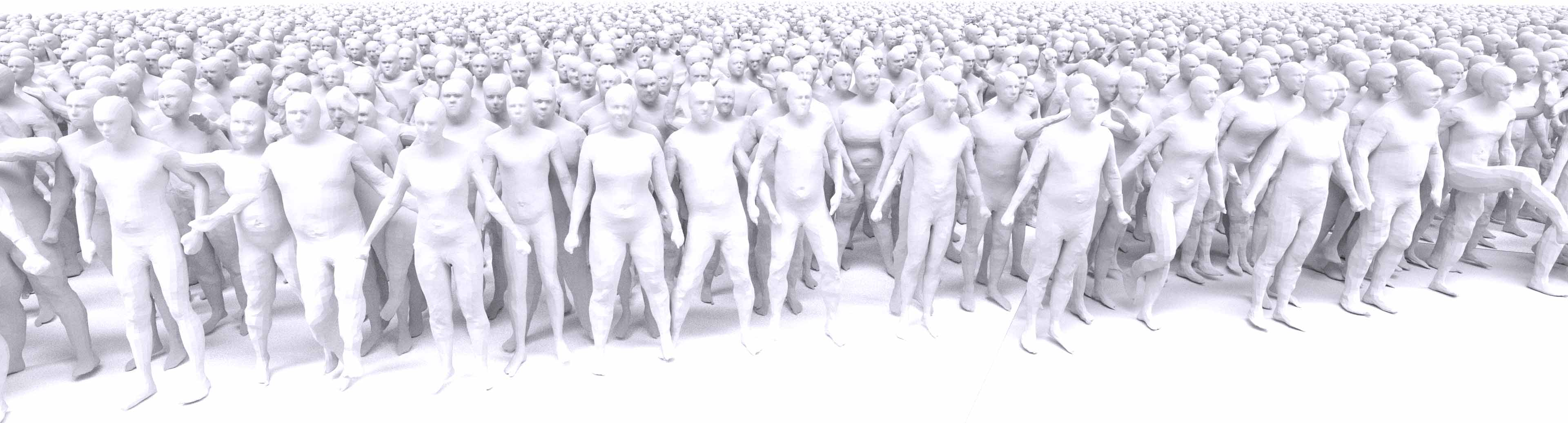}
	\caption{Massive data generation of random 10,000 human models.} \label{fig:massive}
	\vspace{-10pt}
\end{figure*}

\section{Conclusions} In this paper we present a new method for generating random shapes based on a novel 3D shape representation called multi-chart structure. 

The main limitation of our approach is the fact it is restricted to work only with genus-zero (\ie, sphere-type) surfaces. It would be an interesting future work to generalize the method to arbitrary shape topologies, triangle soups and even point clouds. Although opted for conformal mappings, we feel that other parameterization methods (\eg~ area-preserving maps which are used in geometric deep learning \cite{sinha2017surfnet}) can greatly benefit from our multi-chart representation as-well. Furthermore, we could use our representation with other deep generative models such as variational autoencoders (VAEs).

 Currently the reconstruction of the final mesh from the generated charts is done using a fixed template. An interesting future work is to devise more generic ways to reconstruct the final surface mesh from the charts, maybe even incorporate this task into the network. Lastly, we would like to generalize our work to conditional generative models which will allow additional user control of the generated shapes. 

\section{Acknowledgements}
This research was supported in part by the European Research Council (ERC Consolidator Grant, "LiftMatch" 771136), the Israel Science Foundation (Grant No. 1830/17). We would like thank the authors of AtlasNet \cite{Groueix18} and of \cite{litany2017deformable} for sharing their results for comparison.
 
\bibliographystyle{abbrv}

\bibliography{paper}


\appendix

\section{Proofs}\label{appendixA}


To prove Theorem \ref{thm:rigidity} we will prove a more general result dealing with scale-translation rigidity of graphs with respect to per-edge scale and translation. That is, we consider graphs $G=(V_G,E_G)$ where each edge can only be scaled and/or translated, but not rotated. 

%
%
\begin{theorem} \label{thm:graph_rigidity}
	Every generic embedding $q\in\Real^{n\times 3}$ of a 2-connected graph $G=(V_G,E_G)$ with chordless cycles of length at most 4 is unique up to global scale and translation. 
\end{theorem}

This result directly applies to triangulations, which are also graphs, however with less degrees of freedom as only scale and translation of a whole triangle is allowed. 

The general idea of the proof is to first show the theorem for short chordless cycles (Lemma \ref{lemma:1}) and then use it as a building block for proving s-t rigidity of more general graphs (Theorem \ref{lemma:2} and Theorem \ref{thm:graph_rigidity}). 

\begin{Lemma}\label{lemma:1}
	Every generic embedding $q\in\Real^{l\times3}$ of a chordless cycle $C=(V_C,E_C)$ of length $l\leq4$  is unique up to global scale and translation.
\end{Lemma}

\begin{proof}[Proof of lemma \ref{lemma:1}]
	Consider a generic embedding $q\in\Real^{l\times 3}$ of a cycle $C$ of length $l\leq4$. Denote the embeddings of vertices of the chordless cycle by $\{q_i\}_{i=0}^l\subset\Real^3$  where $q_0=q_l$ and the set of vectors connecting neighboring vertices by $u_i=q_{i}-q_{i-1}$, $i\in [l]$. The set $\{u_i\}_{i=1}^l$ satisfies:
	\begin{equation}
	\sum_{i=1}^l u_i=0,
	\end{equation}
	or in matrix form where $\{u_i\}_{i=1}^l$ are the columns of $U\in\Real^{3\times l}$:
	\begin{equation}\label{e:U1_0}
	U\mathbf{1}  = 0.
	\end{equation}
	Since the embedding $q$ is generic, 
	\begin{align}
		\dim \aff \set{q_i}_{i=1}^l &= \dim \mathrm{span} \set{u_i}_{i=1}^l = l-1,
	\end{align}
	where $\aff$ denotes the affine-hull. Therefore the column rank of $U$ is $l-1$ and $\dim \ker U =1 $. 
	
	Now, assume a different embedding $\wt{q}$ such that one edge is fixed, that is w.l.o.g.~$\wt{q}_i=q_i$, $i=0,1$ (\ie, $e_{0,1}$ is fixed). In particular $\wt{u}_1=u_1$.
	Since $\wt{q}$ is an embedding, all vectors are by assumption scaled versions, $\wt{u}_i=\alpha_i u_i$, where $\alpha_i\in\Real$, $i\in[l]$. Furthermore, $\alpha=[\alpha_1,\ldots,\alpha_l]^T$ satisfies $U\alpha=0$. Since $\wt{u}_1=u_1$ we know that $\alpha_1=1$ and since we showed above that $\dim \ker U =1$ we get that $\alpha=\mathbf{1}$. That is, $u_i=\wt{u}_i$, $i\in [l]$. Since $\wt{q}_0=q_0$ we consequently get that $$\wt{q}_j=\wt{q}_0+\sum_{i=1}^j \wt{u}_i=q_0+\sum_{i=1}^j u_i = q_j,$$ for all $j\in [l]$.
	We showed there could be only one generic embedding and therefore the lemma is proved. 
	
	
\end{proof}

\begin{Lemma}\label{lemma:2}
	Having a graph $G=(V_G,E_G)$ and its sub-graph $G'=(V_{G'},E_{G'})$. If there exists a simple cycle $C=(V_C,E_C)$  in G containing an edge from $E_{G'}$ and a vertex from $V_{G}\setminus V_{G'}$, then there exists a chordless cycle $\tilde{C}=(V_{\tilde{C}},E_{\tilde{C}})$ containing an edge from $E_{G'}$ and a vertex from $V_{G}\setminus V_{G'}$.
\end{Lemma}

\begin{proof}[Proof of lemma \ref{lemma:2}]
	If $C$ is chordless we are done. If not we show that a shorter cycle with the same properties can be found: in this case,  there exists an edge $e_{ij}$ with non-consecutive indices. By adding this edge we split the original cycle into two shorter cycles containing $e_{ij}$. If both endpoints of $e_{ij}$ are from $E_{G'}$, keep the cycle that also contains the vertex from $V_{G}\setminus V_{G'}$. Otherwise, keep the cycle containing the edge from $E_{G'}$. In both cases it is guaranteed that the new chosen cycle is shorter and contains an edge from $E_{G'}$ and a vertex from $V_{G}\setminus V_{G'}$. Repeating this process, in a finite number of steps, a chordless cycle satisfying the conditions will be obtained.
\end{proof}

\begin{proof}[Proof of Theorem \ref{thm:graph_rigidity}]
	Let $G$ denote a 2-connected graph with chordless cycles of length at most 4, and $q\in\Real^{n\times 3}$ a generic embedding.  We will show that $q$ is unique up to global scale and translation.
	
	We define an iterative process that grows an s-t rigid subgraph. 
	
	Let $G'$ be a subgraph defined by a set of vertices $V'_G\subset V_G$. First, set $G'$ according to $V'=\set{v_1,v_2}$, where $v_1,v_2\in V_G$ are two adjacent vertices, \ie, $e_{12}\in E_G$. While there is a chordless cycle that contains a vertex $v\in V_G\setminus V'_G$ and an edge in $G'$ add it to $G'$. 
	
	To finish the proof we need to prove: (i) at every iteration of the algorithm $G'$ is s-t rigid; and (ii) when the the algorithm terminates $V'_G=V_G$. 
	
	We start with (i): First, when $V'_G=\set{v_1,v_2}$, $G'$ is s-t rigid by definition. Now given an s-t rigid $G'$, we need to prove that $G'\cup C$ is s-t rigid, where $C$ is a chordless cycle as described above.  Since all chordless cycles in $G$ are of length $\leq 4$, by Lemma \ref{lemma:1} $C$ is s-t rigid. By assumption $G'$ is s-t rigid, and since $G'$ and $C$ share an edge, their union $G'\cup C$ is s-t rigid. 
	
	
	Next, we prove (ii). Assume towards a contradiction that $V_{G}\setminus V_{G'}\neq \emptyset$. Since $G$ is connected there exists an edge $e_{ij}$ with one endpoint $v_i\in V_{G}\setminus V_{G'}$ and the other $v_j\in V_{G'}$. Furthermore, since $G'$ is connected there exists an edge $e_{jk}\in E_{G'}$ with $v_k\in V_{G'}$ (see the inset (a)).
	
	\begin{figure}[H]
		\includegraphics[width=\columnwidth]{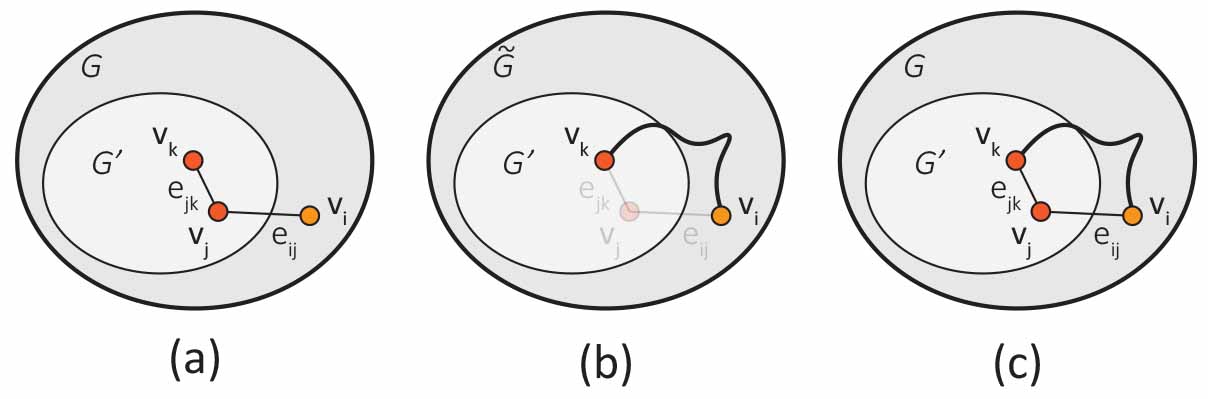}
	\end{figure}

	
	Using the 2-connectedness of $G$, we can exclude $v_j$ to obtain a new connected graph $\tilde{G}$. Since $\tilde{G}$ is connected, there exists a path between $v_i$ and $v_k$ which does not include $v_j$ (inset (b)). Taking this path and completing it with $e_{ij},e_{jk}$ we get a simple cycle containing an edge $e_{jk}\in E_{G'}$ and a vertex from $v_i\in V_{G}\setminus V_{G'}$ (inset (c)). Using Lemma \ref{lemma:2} there exists a chordless cycle $C$ with an edge in $E_{G'}$ and a vertex in $V_G\setminus V'_G$ in contradiction to the fact that the algorithm terminated.

\end{proof}

\begin{proof}[Proof of Proposition \ref{prop:rigidity2}]
	$\impliedby$
	Assume by way of contradiction that there exists two embedding $q,q'\in\Real^{n\times 3}$ that are not related by a global scale and translation, and denote by $r=r(q),r'=r(q')\in \Real^{3\times 3 \times|\mF|}$ the corresponding vertex assignments for all the triangles. WLOG we can assume that $r,r'$ satisfy Equation (\ref{e:P0}) by proper scaling and translating. Furthermore, $r,r'$ satisfy Equation (\ref{e:scale_translation}) as well. This implies that  $r-r'\neq 0$ is in the kernel of the matrix of Equations (\ref{e:scale_translation})-(\ref{e:P0}) which means it is not full rank. 
	
	$\implies$
	Assume by way of contradiction that the linear system  (\ref{e:scale_translation})-(\ref{e:P0}) does not have full column rank. This implies that there exists two different solutions to the system that agree on the first triangle (Equation (\ref{e:P0})). This is a contradiction to the assumption that the triangulation has a unique embedding up to global scale and translation.

\end{proof}

\begin{proof}[Proof of Proposition \ref{thm:st_generic}]

Indeed, let $A\in \Real^{9|\mF|\times (4|\mF|+3|\mV|)}$ be the matrix of the linear system \eqref{e:scale_translation}-\eqref{e:P0}. Since $r_{\scriptscriptstyle{P}}=(r_{\scriptscriptstyle{P},i},r_{\scriptscriptstyle{P},j},r_{\scriptscriptstyle{P},k})$ is a centered-scaled version of $(p_i,p_j,p_k)$ it can be written as $r_P = \alpha_P(p_i,p_j,p_k) + \beta_P$ for some $\alpha_P\in\Real$, $\beta_P\in\Real^3$. Therefore, $\det (A^T A)$ is a polynomial $\pi$ in $\mP\in \Real^{3|\mV|}$ and $\alpha\in\Real^{|\mF|},\beta\in \Real^{3|\mF|}$ and can be written as $\pi(\mP,\alpha,\beta)=\sum_k \tau_k(\mP) \eta_k(\alpha,\beta) $, where $\eta_k(\alpha,\beta)$ are monomials and $\tau_k$ polynomials. If all polynomials $\tau_k$ are the zero polynomials, then $\pi$ is the zero polynomial and $(\mP,T)$ is not s-t rigid for all $\mP$. Otherwise, at-least one $\tau_k$ is not the zero polynomial. Using the fact that a non-zero polynomial is non-zero almost everywhere \cite{caron2005zero} we get that for almost every $\mP$, $\tau_k(\mP)\ne 0$. Fixing such $\mP$ in $\pi$ we have a non-zero polynomial in $\alpha,\beta$ and therefore $\pi(\mP,\alpha,\beta)\ne 0$ for almost all $\alpha,\beta$. 
\end{proof}
 \nocite{heli2018}

\section{Architecture details} \label{appendixB}

{
	\begin{table}[ht]
		\centering 
		\tiny	
		\begin{tabular}{lrlr}
			\toprule
			GENERATOR &       &       &  \\
			\midrule
			&       & input & \multicolumn{1}{l}{output} \\
			&       &       &  \\
			FC    &       & \multicolumn{1}{r}{128} & \multicolumn{1}{l}{4x4x1536} \\
			periodic conv & \multicolumn{1}{l}{3x3} & 4x4x1536 & \multicolumn{1}{l}{4x4x1536} \\
			Relu  &       &       &  \\
			upsample &       & 4x4x1536 & \multicolumn{1}{l}{8x8x1536} \\
			periodic conv & \multicolumn{1}{l}{3x3} & 8x8x1536 & \multicolumn{1}{l}{8x8x768} \\
			Relu  &       &       &  \\
			periodic conv & \multicolumn{1}{l}{3x3} & 8x8x768 & \multicolumn{1}{l}{8x8x768} \\
			Relu  &       &       &  \\
			upsample &       & 8x8x768 & \multicolumn{1}{l}{16x16x768} \\
			periodic conv & \multicolumn{1}{l}{3x3} & 16x16x768 & \multicolumn{1}{l}{16x16x384} \\
			Relu  &       &       &  \\
			periodic conv & \multicolumn{1}{l}{3x3} & 16x16x384 & \multicolumn{1}{l}{16x16x384} \\
			Relu  &       &       &  \\
			upsample &       & 16x16x384 & \multicolumn{1}{l}{32x32x384} \\
			periodic conv & \multicolumn{1}{l}{3x3} & 32x32x384 & \multicolumn{1}{l}{32x32x192} \\
			Relu  &       &       &  \\
			periodic conv & \multicolumn{1}{l}{3x3} & 32x32x192 & \multicolumn{1}{l}{32x32x192} \\
			Relu  &       &       &  \\
			upsample & \multicolumn{1}{l}{3x3} & 32x32x192 & \multicolumn{1}{l}{64x64x192} \\
			periodic conv &       & 64x64x192 & \multicolumn{1}{l}{64x64x96} \\
			Relu  &       &       &  \\
			periodic conv & \multicolumn{1}{l}{3x3} & 64x64x96 & \multicolumn{1}{l}{64x64x96} \\
			Relu  &       &       &  \\
			periodic conv & \multicolumn{1}{l}{1x1} & 64x64x96 & \multicolumn{1}{l}{64x64x48} \\
			symmetry projection layer &       & 64x64x48 & \multicolumn{1}{l}{64x64x48} \\
			landmark consistency &       & 64x64x48 & \multicolumn{1}{l}{64x64x48} \\
			zero mean &       & 64x64x48 & \multicolumn{1}{l}{64x64x48} \\
			\toprule
			DISCRIMINATOR &       &       &  \\
			\midrule
			periodic conv & \multicolumn{1}{l}{1x1} & 64x64x48 & \multicolumn{1}{l}{64x64x96} \\
			LeRelu &       &       &  \\
			periodic conv & \multicolumn{1}{l}{3x3} & 64x64x96 & \multicolumn{1}{l}{64x64x96} \\
			LeRelu &       &       &  \\
			periodic conv & \multicolumn{1}{l}{3x3} & 64x64x96 & \multicolumn{1}{l}{64x64x192} \\
			LeRelu &       &       &  \\
			downsample &       & 64x64x192 & \multicolumn{1}{l}{32x32x192} \\
			periodic conv & \multicolumn{1}{l}{3x3} & 32x32x192 & \multicolumn{1}{l}{32x32x192} \\
			LeRelu &       &       &  \\
			periodic conv & \multicolumn{1}{l}{3x3} & 32x32x192 & \multicolumn{1}{l}{32x32x384} \\
			LeRelu &       &       &  \\
			downsample &       & 32x32x384 & \multicolumn{1}{l}{16x16x384} \\
			periodic conv & \multicolumn{1}{l}{3x3} & 16x16x384 & \multicolumn{1}{l}{16x16x384} \\
			LeRelu &       &       &  \\
			periodic conv & \multicolumn{1}{l}{3x3} & 16x16x384 & \multicolumn{1}{l}{16x16x768} \\
			LeRelu &       &       &  \\
			downsample &       & 16x16x768 & \multicolumn{1}{l}{8x8x768} \\
			periodic conv & \multicolumn{1}{l}{3x3} & 8x8x768 & \multicolumn{1}{l}{8x8x768} \\
			LeRelu &       &       &  \\
			periodic conv & \multicolumn{1}{l}{3x3} & 8x8x768 & \multicolumn{1}{l}{8x8x1536} \\
			LeRelu &       &       &  \\
			downsample &       & 8x8x1536 & \multicolumn{1}{l}{4x4x1536} \\
			periodic conv & \multicolumn{1}{l}{3x3} & 4x4x1536 & \multicolumn{1}{l}{4x4x1536} \\
			LeRelu &       &       &  \\
			periodic conv & \multicolumn{1}{l}{4x4} & 4x4x1536 & \multicolumn{1}{l}{1x1x1536} \\
			LeRelu &       &       &  \\
			FC    &       & 1x1536 & 1 \\
			\bottomrule
		\end{tabular}%
		\caption{Architecture details - humans generating network}
		\label{tab:arch_details_humans}%
	\end{table}%
	
}

{\scriptsize
\begin{table}
	\centering
			\tiny	
	\begin{tabular}{lrlr}
		\toprule
	GENERATOR &       &       &  \\
	\midrule
	&       & input & \multicolumn{1}{l}{output} \\
	&       &       &  \\
	FC    &       & \multicolumn{1}{r}{32} & \multicolumn{1}{l}{4x4x256} \\
	periodic conv & \multicolumn{1}{l}{3x3} & 4x4x256 & \multicolumn{1}{l}{4x4x256} \\
	Relu  &       &       &  \\
	upsample &       & 4x4x256 & \multicolumn{1}{l}{8x8x256} \\
	periodic conv & \multicolumn{1}{l}{3x3} & 8x8x256 & \multicolumn{1}{l}{8x8x128} \\
	Relu  &       &       &  \\
	periodic conv & \multicolumn{1}{l}{3x3} & 8x8x128 & \multicolumn{1}{l}{8x8x128} \\
	Relu  &       &       &  \\
	upsample &       & 8x8x128 & \multicolumn{1}{l}{16x16x128} \\
	periodic conv & \multicolumn{1}{l}{3x3} & 16x16x128 & \multicolumn{1}{l}{16x16x64} \\
	Relu  &       &       &  \\
	periodic conv & \multicolumn{1}{l}{3x3} & 16x16x64 & \multicolumn{1}{l}{16x16x64} \\
	Relu  &       &       &  \\
	upsample &       & 16x16x64 & \multicolumn{1}{l}{32x32x64} \\
	periodic conv & \multicolumn{1}{l}{3x3} & 32x32x64 & \multicolumn{1}{l}{32x32x32} \\
	Relu  &       &       &  \\
	periodic conv & \multicolumn{1}{l}{3x3} & 32x32x32 & \multicolumn{1}{l}{32x32x32} \\
	Relu  &       &       &  \\
	upsample & \multicolumn{1}{l}{3x3} & 32x32x32 & \multicolumn{1}{l}{64x64x32} \\
	periodic conv &       & 64x64x32 & \multicolumn{1}{l}{64x64x16} \\
	Relu  &       &       &  \\
	periodic conv & \multicolumn{1}{l}{3x3} & 64x64x16 & \multicolumn{1}{l}{64x64x16} \\
	Relu  &       &       &  \\
	periodic conv & \multicolumn{1}{l}{1x1} & 64x64x16 & \multicolumn{1}{l}{64x64x12} \\
	symmetry projection layer &       & 64x64x12 & \multicolumn{1}{l}{64x64x12} \\
	landmark consistency &       & 64x64x12 & \multicolumn{1}{l}{64x64x12} \\
	zero mean &       & 64x64x12 & \multicolumn{1}{l}{64x64x12} \\
	\toprule
	DISCRIMINATOR &       &       &  \\
	\midrule
	periodic conv & \multicolumn{1}{l}{1x1} & 64x64x12 & \multicolumn{1}{l}{64x64x16} \\
	LeRelu &       &       &  \\
	periodic conv & \multicolumn{1}{l}{3x3} & 64x64x16 & \multicolumn{1}{l}{64x64x16} \\
	LeRelu &       &       &  \\
	periodic conv & \multicolumn{1}{l}{3x3} & 64x64x16 & \multicolumn{1}{l}{64x64x32} \\
	LeRelu &       &       &  \\
	downsample &       & 64x64x32 & \multicolumn{1}{l}{32x32x32} \\
	periodic conv & \multicolumn{1}{l}{3x3} & 32x32x32 & \multicolumn{1}{l}{32x32x32} \\
	LeRelu &       &       &  \\
	periodic conv & \multicolumn{1}{l}{3x3} & 32x32x32 & \multicolumn{1}{l}{32x32x64} \\
	LeRelu &       &       &  \\
	downsample &       & 32x32x64 & \multicolumn{1}{l}{16x16x64} \\
	periodic conv & \multicolumn{1}{l}{3x3} & 16x16x64 & \multicolumn{1}{l}{16x16x64} \\
	LeRelu &       &       &  \\
	periodic conv & \multicolumn{1}{l}{3x3} & 16x16x64 & \multicolumn{1}{l}{16x16x128} \\
	LeRelu &       &       &  \\
	downsample &       & 16x16x128 & \multicolumn{1}{l}{8x8x128} \\
	periodic conv & \multicolumn{1}{l}{3x3} & 8x8x128 & \multicolumn{1}{l}{8x8x128} \\
	LeRelu &       &       &  \\
	periodic conv & \multicolumn{1}{l}{3x3} & 8x8x128 & \multicolumn{1}{l}{8x8x256} \\
	LeRelu &       &       &  \\
	downsample &       & 8x8x256 & \multicolumn{1}{l}{4x4x256} \\
	periodic conv & \multicolumn{1}{l}{3x3} & 4x4x256 & \multicolumn{1}{l}{4x4x256} \\
	LeRelu &       &       &  \\
	periodic conv & \multicolumn{1}{l}{4x4} & 4x4x256 & \multicolumn{1}{l}{1x1x256} \\
	LeRelu &       &       &  \\
	FC    &       & 1x256 & 1 \\
		\bottomrule
	\end{tabular}%

	\caption{Architecture details - teeth generating network}
	\label{tab:arch_details_teeth}%
		\vspace{300pt}
\end{table}%
}

\end{document}